\documentclass{article}
\usepackage{ijcai26}

\usepackage{times}
\usepackage{soul}
\usepackage{url}
\usepackage[hidelinks]{hyperref}
\usepackage[utf8]{inputenc}
\usepackage[small]{caption}
\usepackage{graphicx}
\usepackage{amsmath}
\usepackage{amsthm}
\usepackage{booktabs}
\usepackage{algorithm}
\usepackage{algorithmic}
\usepackage{float}
\usepackage[switch]{lineno}
\usepackage{amssymb}
\usepackage{mathtools}
\usepackage{tikz}
\usepackage{pgfplots}
\usepackage{subcaption}
\pgfplotsset{compat=1.18}
\usetikzlibrary{arrows.meta, positioning, fit}
\usetikzlibrary{positioning, arrows.meta, decorations.pathreplacing, calc}

\newcommand{\alghl}[1]{{\color{blue} #1}} 
\newcommand{\KP}[1]{{\color{red} #1}}

\newtheorem{theorem}{Theorem}
\newtheorem{lemma}{Lemma}

\title{Online Risk-Averse Planning in POMDPs Using Iterated CVaR Value Function}

\author{
Yaacov Pariente$^1$\and
Vadim Indelman$^{2,3}$\\
\affiliations
$^1$Faculty of Mathematics\\
$^2$Stephen B. Klein Faculty of Aerospace Engineering\\
$^3$Faculty of Data and Decision Sciences\\
Technion - Israel Institute of Technology\\
\emails
yaacovp@campus.technion.ac.il,
vadim.indelman@technion.ac.il
}

\begin{document}

\maketitle
\raggedbottom

\begin{abstract}
    We study risk-sensitive planning under partial observability using the dynamic risk measure Iterated Conditional Value-at-Risk (ICVaR). A policy evaluation algorithm for ICVaR is developed with finite-time performance guarantees that do not depend on the cardinality of the action space. Building on this foundation, three widely used online planning algorithms—Sparse Sampling, Particle Filter Trees with Double Progressive Widening (PFT-DPW), and Partially Observable Monte Carlo Planning with Observation Widening (POMCPOW)—are extended to optimize the ICVaR value function rather than the expectation of the return. Our formulations introduce a risk parameter $\alpha$, where $\alpha = 1$ recovers standard expectation-based planning and $\alpha < 1$ induces increasing risk aversion. For ICVaR Sparse Sampling, we establish finite-time performance guarantees under the risk-sensitive objective, which further enable a novel exploration strategy tailored to ICVaR. Experiments on benchmark POMDP domains demonstrate that the proposed ICVaR planners achieve lower tail risk compared to their risk-neutral counterparts.
\end{abstract}

\section{Introduction}
Autonomous agents have emerged as pivotal components in various fields such as robotics, healthcare, education, manufacturing, and financial services. The deployment of these agents in real-world environments necessitates rigorous adherence to safety protocols to ensure their reliable and ethical operation. One of the key challenges in this domain is partial observability, where agents must make decisions with incomplete information about their surroundings. Addressing this challenge often involves employing robust decision-making frameworks like risk-averse Partially Observable Markov Decision Processes (POMDPs), which incorporate risk metrics to optimize performance while mitigating potential hazards. 

The Conditional Value at Risk (CVaR) \cite{rockafellar2000optimization} is a widely used risk measure that facilitates the optimization of the upper tail of the cost's distribution. The dual representation of CVaR \cite{artzner1999coherent} enables its interpretation as the worst-case expectation of the cost \cite{chow2015risk}, thereby motivating its use for risk-averse decision making. CVaR is also a coherent risk measure with desirable properties for safe planning \cite{majumdar2020should}, and its estimators have performance guarantees that ensure their reliability in practice \cite{brown2007large,pmlr-v97-thomas19a}.

Integration of CVaR into MDPs can be achieved through several methodologies, such as defining the value function as the CVaR of the return \cite{chow2015risk} or using CVaR as an optimization constraint \cite{chow2014algorithms}. A difficulty in optimizing the CVaR of the return, in contrast to the expectation of the return, is that CVaR recursive equations \cite{pflug2000some} are hard to solve \cite{chow2015risk}. An alternative method involves utilizing the dynamic risk measure Iterated CVaR (ICVaR) \cite{Mary2004} to define the value function, thereby ensuring agent safety \cite{du2022provably}. For a detailed discussion on the distinction between CVaR and ICVaR, see \cite{du2022provably}. Risk-averse POMDP planning was studied with general dynamic risk objectives, where CVaR is a special case \cite{ahmadi2020risk}.

Monte Carlo Tree Search (MCTS) is a widely used algorithm for decision-making in sequential fully observable problems. It combines principles of tree search and stochastic simulation to iteratively explore and evaluate potential action sequences \cite{coulom2006efficient}. Several variants of MCTS have been developed to address the challenges of planning in POMDPs. Among these, POMCP (Partially Observable Monte Carlo Planning) represents a direct extension of MCTS to the POMDP setting, effectively handling partial observability through particle-based belief state approximations \cite{silver2010monte}. Extensions such as POMCPOW, POMCP-DPW, and PFT-DPW further generalize POMCP by enabling planning in continuous action and observation spaces, thereby broadening its applicability to complex domains \cite{sunberg2018online}.

We present a framework for risk-averse online planning in partially observable domains using the ICVaR dynamic risk measure. Our contributions are fourfold. First, we develop a policy evaluation algorithm for ICVaR with finite-time performance guarantees. Second, we introduce ICVaR Sparse Sampling, extending sparse sampling to optimize ICVaR and deriving corresponding finite-time guarantees. Third, we adapt POMCPOW and PFT-DPW to optimize ICVaR. Fourth, leveraging the finite-time guarantees, we propose a novel exploration strategy specifically designed for ICVaR objectives, replacing standard MCTS exploration which is tailored to expectation-based value functions. To our knowledge, this work is the first to develop online planning methods for risk-averse POMDPs.

\section{Preliminaries}
\subsection{Partially Observable Markov Decision Process}\label{sec:POMDP_preliminaries}
A finite horizon POMDP is defined as a tuple $M\triangleq (X,A,Z,T,O,c, \gamma, b_0)$, where $X,A,Z$ are the state, action and observation spaces respectively. The state transition model $T(x_{t+1}|x_t,a_t) \triangleq P(x_{t+1}|x_t,a_t)$ is the probability of moving from state $x_t$ to state $x_{t+1}$, given the agent performed the action $a_t$. The observation model $O(z_t|x_t) \triangleq P(z_t|x_t)$ is the probability of observing $z_t$, given the true state is $x_t$. The cost function is defined by $c:B\times A \rightarrow \mathbb{R}$, where $B$ is the set of all beliefs. 

Due to partial observability, the agent maintains a probability distribution over the current state given the previous observations and actions, known as the belief. Formally, the belief is defined by $b(x_t) \triangleq P(x_t|H_t)$ for $x_t\in X$ and history $H_t \triangleq \{z_{1:t},a_{0:t-1},b_0\}$, and can be expressed recursively by the following equation $b(x_t) =\eta_tP(z_t|x_t)\int_{x_{t-1}\in X}P(x_t|x_{t-1},a_{t-1})b(x_{t-1})dx_{t-1}$, where $\eta_t$ is a normalization constant. 

A policy $a_t=\pi_t(b_t)$ is a mapping from a belief to an action at time t. The cost of action $a_t$ after seeing belief $b_t$ is $c(b_t,a_t) \triangleq E_{x \sim b_t}[c_x(x,a_t)]$ such that $|c_x(x,a_t)|\leq R_{max}$. The cost for the finite horizon  $T\in\mathbb{N}$, also known as the return, is $R_{t:T} \triangleq\sum_{\tau=t}^T \gamma^{\tau-t}c(b_\tau,a_\tau)$, which is a measure of the agent’s success at time t.

In a standard setting, the value function  is defined, for a given policy sequence $\pi\triangleq\{\pi_t, \ldots \pi_{t+T}\}$  and planning horizon of $T$ time steps, as the expected return,
\begin{equation}\label{Eq:regular_V_function}
	V^\pi (b_t)\triangleq\mathbb{E}[R_{t:T}|b_t,\pi]=\sum_{\tau=t}^T \gamma^{\tau-k}\mathbb{E}[c(b_\tau,a_\tau)|b_t,\pi],
\end{equation}
and the Q function is 
\begin{equation}\label{Eq:regular_Q_function}
	Q^\pi (b_t,a_t) \triangleq c(b_t,a_t)+ \gamma \mathbb{E}_{z_{t+1}}[ V^\pi(b_{t+1})|b_t,a_t].
\end{equation}
In \eqref{Eq:regular_V_function}, $a_t=\pi_t(b_t)$, and in \eqref{Eq:regular_Q_function}, $a_t \in A$ is some action.

\subsection{Conditional Value at Risk}\label{sec:cvar_preliminaries}
Let \( Y \) be a random variable with $\mathbb{E}|Y|<\infty$ and define $F(y) \triangleq P(Y\leq y)$. The value at risk at confidence level $\alpha \in (0,1)$ is the $1-\alpha$ quantile of $Y$, i.e., $VaR_\alpha (Y) \triangleq \sup\{y\in \mathbb{R}:F(y) \leq 1-\alpha\}$. The conditional value at risk (CVaR) at confidence level $\alpha$ is defined as \cite{rockafellar2000optimization}
\begin{equation}
	CVaR_\alpha(Y):=\inf_{w\in \mathbb{R}}\Bigl\{w+\frac{1}{\alpha}\mathbb{E}[(Y-w)^+]\Bigr\},
\end{equation}
where $(y)^+ \triangleq \max(y,0)$. For a continuous $F$, it holds that \cite{pflug2000some}
\begin{equation}\label{def:cvar_as_conditional_expectation}
	CVaR_\alpha(Y)=\mathbb{E}[Y|Y>VaR_\alpha(Y)]=\frac{1}{\alpha}\int_{1-\alpha}^1 F^{-1}(v)dv.
\end{equation}
Let $Y_i \overset{\text{iid}}{\sim} F$ for $i\in\{1,\dots,n\}$. Denote by
\begin{equation}\label{eq:brown_cvar_estimator}
	\hat{C}_\alpha(Y) \triangleq \hat{C}_\alpha(\{Y_i\}^{n}_{i=1}) \triangleq \inf_{y\in \mathbb{R}} \Bigl\{y+\frac{1}{n\alpha} \sum_{i=1}^n(Y_i-y)^+\Bigr\}
\end{equation}
the estimate of $CVaR_\alpha(Y)$ \cite{brown2007large}.

\eqref{eq:brown_cvar_estimator} can be expressed as
\begin{equation}\label{eq:cvar_estimator_sorted_sample}
	\hat{C}_\alpha(Y)=\frac{1}{\lceil \alpha n \rceil} \sum_{i=1}^{\lceil \alpha n \rceil} Y^{(i)},
\end{equation}
where $Y^{(i)}$ is the $i$th order statistic of $Y_1,\dots,Y_n$ in ascending order, i.e., the mean of the $\lceil \alpha n \rceil$ smallest samples.

\section{Problem Formulation}
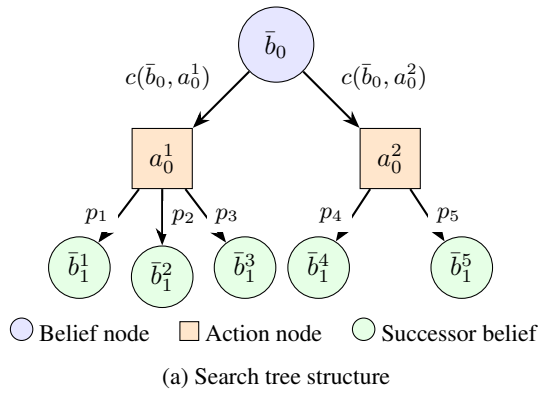
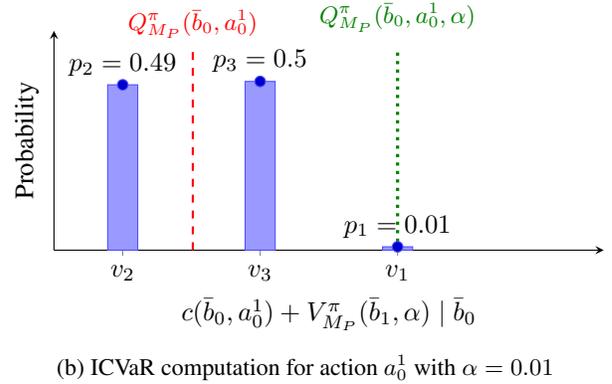
\begin{figure*}[t]
	\centering
	\begin{minipage}[b]{0.45\textwidth}
		\centering
		\begin{tikzpicture}[
			state/.style={circle, draw, minimum size=1cm, fill=blue!10},
			action/.style={rectangle, draw, minimum size=0.8cm, fill=orange!20},
			successor/.style={circle, draw, minimum size=0.8cm, fill=green!10},
			edge_label/.style={midway, fill=white, font=\small},
			arrow/.style={-{Stealth}, thick},
			brace/.style={decorate, decoration={brace, amplitude=5pt, raise=2pt}},
			]
			
			\node[state] (s) {$\bar{b}_0$};
			
			\node[action, below left=0.75cm and 0.75cm of s] (a1) {$a_0^1$};
			\node[action, below right=0.75cm and 0.75cm of s] (a2) {$a_0^2$};
			
			\draw[arrow] (s) -- (a1) node[edge_label, above left] {$c(\bar{b}_0, a_0^1)$};
			\draw[arrow] (s) -- (a2) node[edge_label, above right] {$c(\bar{b}_0, a_0^2)$};
			
			\node[successor, below left=0.75cm and 0.4cm of a1] (s1_1) {$\bar{b}_1^1$};
			\node[successor, below=0.75cm of a1] (s1_2) {$\bar{b}_1^2$};
			\node[successor, below right=0.75cm and 0.4cm of a1] (s1_3) {$\bar{b}_1^3$};
			
			\draw[arrow] (a1) -- (s1_1) node[edge_label, left] {$p_1$};
			\draw[arrow] (a1) -- (s1_2) node[edge_label, right] {$p_2$};
			\draw[arrow] (a1) -- (s1_3) node[edge_label, right] {$p_3$};
			
			\node[successor, below left=0.75cm and 0.25cm of a2] (s2_1) {$\bar{b}_1^4$};
			\node[successor, below right=0.75cm and 0.25cm of a2] (s2_2) {$\bar{b}_1^5$};
			
			\draw[arrow] (a2) -- (s2_1) node[edge_label, left] {$p_4$};
			\draw[arrow] (a2) -- (s2_2) node[edge_label, right] {$p_5$};
			
			\node[below=3cm of s, align=left, font=\small] {
				\tikz\node[state, minimum size=0.25cm]{}; Belief node \quad
				\tikz\node[action, minimum size=0.25cm]{}; Action node \quad
				\tikz\node[successor, minimum size=0.25cm]{}; Successor belief
			};
			
		\end{tikzpicture}
		\subcaption{Search tree structure}
		\label{fig:icvar_tree_structure}
	\end{minipage}
	\hfill
	\begin{minipage}[b]{0.5\textwidth}
		\centering
		\begin{tikzpicture}
			\begin{axis}[
				width=\textwidth,
				height=4.5cm,
				xlabel={$c(\bar{b}_0, a_0^1) + V^\pi_{M_P}(\bar{b}_1, \alpha) \mid \bar{b}_0$},
				ylabel={Probability},
				ymin=0, ymax=0.65,
				xmin=-0.5, xmax=3.5,
				xtick={0, 1, 2},
				xticklabels={$v_2$, $v_3$, $v_1$},
				ytick=\empty,
				axis lines=left,
				enlargelimits=false,
				clip=false,
				]
				
				\addplot+[ybar, bar width=0.4cm, fill=blue!40, draw=blue!70] 
				coordinates {(0, 0.49) (1, 0.5) (2, 0.01)};
				
				\node[above] at (axis cs:0, 0.49) {$p_2=0.49$};
				\node[above] at (axis cs:1, 0.5) {$p_3=0.5$};
				\node[above] at (axis cs:2, 0.01) {$p_1=0.01$};
				
				\addplot[dashed, thick, red] coordinates {(0.51, 0) (0.51, 0.6)};
				\node[above, red, font=\small] at (axis cs:0.51, 0.6) {$Q^\pi_{M_P}(\bar{b}_0, a_0^1)$};
				
				\addplot[dotted, very thick, green!50!black] coordinates {(2, 0) (2, 0.6)};
				\node[above, green!50!black, font=\small, align=center] at (axis cs:2, 0.62) {$Q^\pi_{M_P}(\bar{b}_0, a_0^1, \alpha)$};
			\end{axis}
		\end{tikzpicture}
		\subcaption{ICVaR computation for action $a_0^1$ with $\alpha=0.01$}
		\label{fig:icvar_pdf}
	\end{minipage}
	
	\caption{Illustration of ICVaR-based Q-value computation in a search tree. (a) At each action node, the Q-value is computed as the immediate cost $c(\bar{b}_0,a)$ plus the CVaR of the value distribution over successor beliefs. (b) Example PDF for action $a_0^1$ with $p_1=0.01$, $p_2=0.49$, $p_3=0.5$ and $v_2 < v_3 < v_1$. The dashed red line shows the risk-neutral Q-value (expected value), while the dotted green line shows the risk-sensitive $Q(\bar{b}_0, a_0^1, \alpha)$ at the 0.99 quantile (upper tail). In this example, $Q^\pi_{M_P}(\bar{b}_0, a_0^1,\alpha)$ evaluates the value associated with the most adverse belief, even though such a belief may be rarely sampled. In contrast, $Q^\pi_{M_P}(\bar{b}_0, a_0^1)$ aggregates contributions from all beliefs irrespective of their associated cost, and therefore does not account for their relative risk or severity.}
	\label{fig:icvar_tree}
\end{figure*}

Our approach employs the dynamic CVaR risk measure as the value function. Recursive forms of static CVaR risk measure exist \cite{pflug2000some}, and can be used to learn optimal policies using Bellman equations. However, these formulas are hard to compute in practice \cite{chow2015risk}. Hence, we choose a computationally cheaper alternative and optimize the ICVaR value function \cite{chu2014markov}. Formally, let $\alpha \in (0,1)$ and policy $\pi$. Define the ICVaR action-value function
\begin{equation}\label{eq:icvar_q_function_def}
	\begin{aligned}
		Q^{\pi}_{M, t}(b_t, a, \alpha) \! &\triangleq \! c(b_t, a)  \\
		&+ \gamma CVaR_\alpha^P[V_{M, t+1}^\pi(b_{t+1}, \alpha)|b_t, a, \pi],
	\end{aligned}
\end{equation}
for 
\begin{equation}
	\begin{aligned}
		&CVaR_\alpha^P[V_{M}^\pi(b_{t+1}, \alpha)|b_t, a, \pi] \! \\
		&\triangleq \underset{b_{t+1} \sim P(\cdot|b_{t}, a)}{CVaR}[V_{M}^\pi(b_{t+1}, \alpha)].
	\end{aligned}
\end{equation}
The value function is defined directly using the action-value function by 
\begin{equation}\label{eq:icvar_v_function_def}
	\begin{aligned}
		V_{M, t}^\pi(b_t, \alpha) \triangleq Q^{\pi}_{M,t}(b_t, \pi(b_t), \alpha),
	\end{aligned}
\end{equation}
such that $V_{M, t}^\pi(b_t,\alpha)=0$ for all $t>T$. Figure~\ref{fig:icvar_tree} illustrates the construction of a search tree under ICVaR optimization and highlights the distinction between ICVaR-based and standard expectation-based optimization.

In order to estimate the theoretical action-value function, we consider a particle belief MDP (PB-MDP) setting. Formally, denote by $M_P \triangleq (\Sigma, A, \tau, \rho, \gamma)$ the PB-MDP that is defined with respect to the POMDP $M$ \cite{lim2023optimality} and $N_p \in \mathbb{N}$, where 
\begin{itemize}
	\item $\Sigma \triangleq \{\bar{b}:\bar{b}=\{(x_i,w_i)\}_{i=1}^{N_p},x_i\in X, \forall i,w_i\geq 0,\exists i\text{ such that } w_i>0\}$ is the state space over the particle beliefs.
	\item $A$ is the action space as defined in the POMDP $M$.
	\item $\tau(\bar{b}_{t+1}|\bar{b}_{t}, a)$ is the belief transition probability, for $a\in A,\bar{b}\in \Sigma$.
	\item $\rho(\bar{b},a) \triangleq \frac{\sum_{j=1}^{N_p} w_i c(x_i,a)}{\sum_{i=1}^{N_p} w_i}$ is the state dependent cost.
	\item $\gamma$ as defined in the POMDP $M$.
\end{itemize}
The action-value function with respect to the PB-MDP is defined by \begin{equation}
	\begin{aligned}
		Q^{\pi}_{M_P, t}(\bar{b}_t, a, &\alpha) \! \triangleq \! c(\bar{b}_t, a) \\
		&+ \gamma CVaR_\alpha^{M_P}[V_{M_P, t+1}^\pi(\bar{b}_{t+1}, \alpha)|\bar{b}_t, a, \pi],
	\end{aligned}
	\label{eq:q_hat_estimation}
\end{equation}
and value function is defined directly using the action-value function by 
\begin{equation}\label{Eq:cvar_Q_function_def}
	\begin{aligned}
		V_{M_P, t}^\pi(\bar{b}_t, \alpha) \triangleq Q^{\pi}_{M_P,t}(\bar{b}_t, \pi(\bar{b}_t), \alpha),
	\end{aligned}
\end{equation}
such that $V_{M_P, t}^\pi(b_t,\alpha)=0$ for all $t>T$. 

The goal of the paper is to perform planning within this risk-averse framework. 

\section{ICVaR Policy Evaluation}\label{sec:icvar_policy_evaluation}
The first contribution of this paper is Algorithm~\ref{alg:icvar_sparse_sampling_policy_evaluation}, a policy evaluation algorithm that estimates the ICVaR PB-MDP action-value function for a given policy $\pi$.  The algorithm recursively estimates the value function \eqref{eq:q_hat_estimation} by sampling $N_b$ successor beliefs from each belief-action pair and aggregating their values using the CVaR statistic. This captures the tail risk of the successor value function's distribution under policy $\pi$. Unlike expectation-based policy evaluation, which averages over successor values, ICVaR policy evaluation focuses on the worst $\alpha$-fraction of outcomes.

At time $t$, the agent has access to a particle belief $\bar{b}_t \triangleq \{(x_t^i,w_t^i)\}_{i=1}^{N_p}$, where $x_t^i\in X$ is the $i$th particle state and $w_t^i$ is the particle's weight. The algorithm simulates $N_b$ successor beliefs $\{\bar{b}_{t+1}^i\}_{i=1}^{N_b}$ and estimates their values, yielding
\begin{align}\label{eq:Qest}
	\hat{Q}_{M_P, t}^\pi(\bar{b}_t,a,\alpha) &\triangleq \sum_{j=1}^{N_p} \tilde{w}_t^jc(x_t^j, a) \\
	&+ \hat{C}_\alpha(\{\hat{V}_{M_P, t+1}^\pi(\bar{b}_{t+1}^i)\}_{i=1}^{N_b}),\\
	\hat{V}_{M_P, t}^\pi(\bar{b}_t,\alpha) &\triangleq \hat{Q}_{M_P, t}^{\pi_t}(\bar{b}_t,\pi_t(\bar{b}_t),\alpha), \label{eq:Vest}
\end{align}
where $\tilde{w}_t^j \triangleq w_t^j / \sum_{k=1}^{N_p} w_t^k$ are the normalized weights. The states and weights at time $t$ are treated as constants, while those at time $t+1$ are random variables obtained via simulation. Algorithm~\ref{alg:icvar_sparse_sampling_policy_evaluation} implements \eqref{eq:q_hat_estimation}, and Figure~\ref{fig:ss_policy_eval_illustration} illustrates the belief tree expansion.
\begin{figure}[H]
	\centering
	\begin{tikzpicture}[
		scale=0.7, every node/.style={scale=0.7},
		state/.style={circle, draw, minimum size=1cm, fill=blue!10},
		action/.style={rectangle, draw, minimum size=0.7cm, fill=orange!20},
		successor/.style={circle, draw, minimum size=0.9cm, fill=green!10},
		arrow/.style={-{Stealth}, thick},
		plan_label/.style={font=\scriptsize, fill=white, inner sep=1pt},
		]
		
		\node[state] (b0) {$\bar{b}_0$};
		
		\node[action, below=0.6cm of b0] (a0) {$\pi(\bar{b}_0)$};
		
		\draw[arrow] (b0) -- (a0);
		
		\node[successor, below left=0.5cm and 1.2cm of a0] (b1_1) {$\bar{b}_1^1$};
		\node[successor, below left=0.5cm and 0.1cm of a0] (b1_2) {$\bar{b}_1^2$};
		\node[successor, below right=0.5cm and 0.1cm of a0] (b1_3) {$\bar{b}_1^3$};
		\node[successor, below right=0.5cm and 1.2cm of a0] (b1_4) {$\bar{b}_1^4$};
		
		\draw[arrow] (a0) -- (b1_1);
		\draw[arrow] (a0) -- (b1_2);
		\draw[arrow] (a0) -- (b1_3);
		\draw[arrow] (a0) -- (b1_4);
		
		\node[action, below=0.5cm of b1_1] (a1) {$\pi(\bar{b}_1^1)$};
		\draw[arrow] (b1_1) -- (a1);
		
		\node[action, below=0.5cm of b1_2] (a2) {$\pi(\bar{b}_1^2)$};
		\draw[arrow] (b1_2) -- (a2);
		
		\node[action, below=0.5cm of b1_3] (a3) {$\pi(\bar{b}_1^3)$};
		\draw[arrow] (b1_3) -- (a3);
		
		\node[action, below=0.5cm of b1_4] (a4) {$\pi(\bar{b}_1^4)$};
		\draw[arrow] (b1_4) -- (a4);
		
		\node[below=0.02cm of a1] {$\vdots$};
		\node[below=0.02cm of a2] {$\vdots$};
		\node[below=0.02cm of a3] {$\vdots$};
		\node[below=0.02cm of a4] {$\vdots$};
		
		\node[below=2.6cm of a0, align=left, font=\normalsize] {
			\tikz\node[state, minimum size=0.5cm]{}; Initial belief \quad
			\tikz\node[action, minimum size=0.5cm]{}; Action node \quad
			\tikz\node[successor, minimum size=0.5cm]{}; Sampled belief
		};
		
	\end{tikzpicture}
	\caption{ICVaR policy evaluation belief tree 
		with $N_b=4$.}
	\label{fig:ss_policy_eval_illustration}
\end{figure}
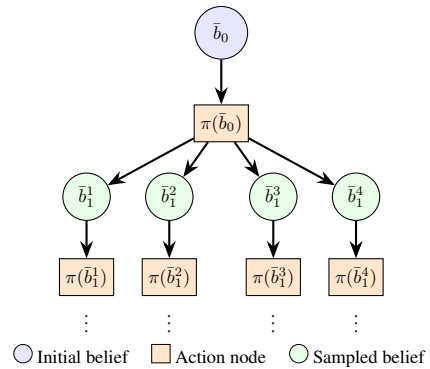
The time complexity of policy evaluation is $O\bigl((N_pN_b)^{T-t}\bigr)$, as only one action is evaluated per belief node.

\section{ICVaR Sparse Sampling}\label{sec:icvar_sparse_sampling}
Building on the policy evaluation framework, we now consider the problem of computing an optimal policy. Unlike policy evaluation, which follows a fixed policy $\pi$, sparse sampling \cite{Kearns02jml} computes the optimal action at each belief by evaluating all actions in $A$ and selecting the one with minimum estimated action-value. This exhaustive search over actions enables optimal planning with finite-time guarantees but increases computational cost. In contrast to expectation-based sparse sampling, which selects actions minimizing expected cost, ICVaR sparse sampling selects actions that minimize the CVaR of the successor value function's distribution, yielding risk-sensitive policies. Figure~\ref{fig:icvar_sparse_sampling_tree} illustrates the search tree structure and the CVaR computation at each action node.

The optimal value functions are defined as $V^*_{M_P,t}(\bar{b}_t, \alpha) \triangleq \min_{a \in A} Q^*_{M_P,t}(\bar{b}_t, a, \alpha)$ and $Q^*_{M_P,t}(\bar{b}_t, a, \alpha)$ denotes the action-value function under the optimal policy $\pi^*$. The optimal policy, with respect to the estimated action-value function, can be computed via Bellman optimality as follows:
\begin{equation}\label{eq:estimated_optimal_policy}
	\pi^*(\bar{b}_t)=\min_{a_t\in A} \hat{Q}_{M_P}^{\pi^*}(\bar{b}_t,a_t,\alpha),
\end{equation}
where $\hat{Q}_{M_P}^{\pi^*}(\bar{b}_t,a_t,\alpha)$ is updated according to \eqref{eq:Qest} and \eqref{eq:Vest}. Algorithm~\ref{alg:icvar_sparse_sampling} presents the ICVaR Sparse Sampling procedure for computing the optimal policy.

\begin{algorithm}[H]
	\caption{ICVaR Policy Evaluation-$\pi$}
	\label{alg:icvar_sparse_sampling_policy_evaluation}
	\begin{algorithmic}[1]
		\STATE \textbf{Global Variables:} $\gamma$, $N_b$, $N_p$, $T$, $\pi$, $\alpha$
		\STATE
		\STATE \textbf{Function} \textsc{EstimateV$^{\pi}$}($\bar{b}, \alpha, t$)
		\STATE \textbf{Input:} Particle belief set $\bar{b} = \{(x^i_t, w^i_t)\}$, risk level $\alpha$, depth $t$
		\STATE \textbf{Output:} A scalar $\hat{V}^\pi_{M_P}(\bar{b}_t, \alpha)$ that is an estimate of $V^\pi_{M_P}(\bar{b}_t)$
		\IF{$t \geq T$}
		\STATE \textbf{return} 0
		\ENDIF
		\STATE $\hat{Q}^\pi_{M_P, t}(\bar{b}, \pi(\bar{b}), \alpha) \gets \textsc{EstimateQ}^\pi(\bar{b}, \pi(\bar{b}), \alpha, t)$
		\STATE \textbf{return} $\hat{V}^\pi_{M_P, t}(\bar{b}) \gets \hat{Q}^\pi_{M_P, t}(\bar{b}, \pi(\bar{b}), \alpha)$
		\STATE \textbf{End Function}
		\STATE
		\STATE \textbf{Function} \textsc{EstimateQ$^\pi$}($\bar{b}, a, \alpha, t$)
		\STATE \textbf{Input:} Particle belief set $\bar{b} = \{(x_i, w_i)\}$, action $a$, risk level $\alpha$, depth $t$
		\STATE \textbf{Output:} A scalar $\hat{Q}^\pi_{M_P, t}(\bar{b}, a, \alpha)$ that is an estimate of $Q^\pi_{M_P, t}(\bar{b}, a, \alpha)$ at time $t$
		\FOR{$i = 1$ to $N_b$}
		\STATE $\bar{b}'_i, \rho \gets \textsc{GenPF}(\bar{b}, a)$ \hfill (Alg.~\ref{alg:common-procedures})
		\STATE $\hat{V}^\pi_{M_P, t+1}(\bar{b}'_i) \gets \textsc{EstimateV}^\pi(\bar{b}'_i, \alpha, t + 1)$
		\ENDFOR
		\STATE \textbf{return} $\hat{Q}^\pi_{M_P, t}(\bar{b}, a, \alpha) \gets \rho + \gamma \, \hat{C}_\alpha\left(\{\hat{V}^\pi_{M_P, t+1}(\bar{b}'_i)\}^{N_b}_{i=1}\right)$
		\STATE \textbf{End Function}
	\end{algorithmic}
\end{algorithm}

The time complexity of computing the optimal policy $\pi^*$ is $O\Bigl{(}(N_pN_b|A|)^{T-t}\Bigr{)}$, assuming state-dependent costs, making it suitable only for problems that require a short planning horizon. Moreover, the algorithm requires enumeration over the action space $A$, limiting its applicability to problems with discrete and moderately-sized action spaces.

\begin{algorithm}[H]
	\caption{ICVaR Sparse Sampling}
	\label{alg:icvar_sparse_sampling}
	\begin{algorithmic}[1]
		\STATE \textbf{Global Variables:} $\gamma$, $N_b$, $N_p$, $T$, $\alpha$, $A$
		\STATE
		\STATE \textbf{Function} \textsc{EstimateV$^*$}($\bar{b}, \alpha, t$)
		\STATE \textbf{Input:} Particle belief set $\bar{b} = \{(x_t^i, w^i_t)\}$, risk level $\alpha$, depth $t$
		\STATE \textbf{Output:} A scalar $\hat{V}^*_{M_P}(\bar{b}_t, \alpha)$ that is an estimate of $V^*_{M_P}(\bar{b}_t, \alpha)$
		\IF{$t \geq T$}
		\STATE \textbf{return} 0
		\ENDIF
		\STATE $a^* \gets \arg\min_{a \in A} \textsc{EstimateQ}^*(\bar{b}, a, \alpha, t)$
		\STATE \textbf{return} $\hat{V}^*_{M_P, t}(\bar{b}, \alpha) \gets \hat{Q}^*_{M_P, t}(\bar{b}, a^*, \alpha)$
		\STATE \textbf{End Function}
		\STATE
		\STATE \textbf{Function} \textsc{EstimateQ$^*$}($\bar{b}, a, \alpha, t$)
		\STATE \textbf{Input:} Particle belief set $\bar{b} = \{(x_i, w_i)\}$, action $a$, risk level $\alpha$, depth $t$
		\STATE \textbf{Output:} A scalar $\hat{Q}^*_{M_P, t}(\bar{b}, a, \alpha)$ that is an estimate of $Q^*_{M_P, t}(\bar{b}, a, \alpha)$
		\FOR{$i = 1$ to $N_b$}
		\STATE $\bar{b}'_i, \rho \gets \textsc{GenPF}(\bar{b}, a)$ \hfill (Alg.~\ref{alg:common-procedures})
		\STATE $\hat{V}^*_{M_P, t+1}(\bar{b}'_i, \alpha) \gets \textsc{EstimateV}^*(\bar{b}'_i, \alpha, t + 1)$
		\ENDFOR
		\STATE \textbf{return} $\hat{Q}^*_{M_P, t}(\bar{b}, a, \alpha) \gets \rho + \gamma \, \hat{C}_\alpha\left(\{\hat{V}^*_{M_P, t+1}(\bar{b}'_i, \alpha)\}^{N_b}_{i=1}\right)$
		\STATE \textbf{End Function}
	\end{algorithmic}
\end{algorithm}

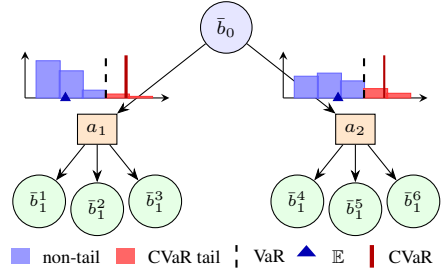
\begin{figure}[t]
	\centering
	\begin{tikzpicture}[
		state/.style={circle, draw, minimum size=0.5cm, fill=blue!10, font=\scriptsize},
		action/.style={rectangle, draw, minimum size=0.4cm, fill=orange!20, font=\scriptsize},
		successor/.style={circle, draw, minimum size=0.35cm, fill=green!10, font=\tiny},
		arrow/.style={-{Stealth}, thin},
		]

		\node[state] (b0) {$\bar{b}_0$};

		\node[action, below left=0.9cm and 1.2cm of b0] (a1) {$a_1$};
		\node[action, below right=0.9cm and 1.2cm of b0] (a2) {$a_2$};

		\draw[arrow] (b0) -- (a1);
		\draw[arrow] (b0) -- (a2);

		\node[above=0.05cm of a1, anchor=south] (dist1) {
			\begin{tikzpicture}
				\begin{axis}[
					width=3.5cm, height=2.2cm,
					axis lines=left,
					xtick=\empty, ytick=\empty,
					xmin=0, xmax=5, ymin=0, ymax=0.6,
					clip=false,
					]
					\addplot[ybar, bar width=0.32cm, fill=blue!40, draw=blue!60]
					coordinates {(0.8, 0.48) (1.6, 0.35) (2.4, 0.10)};
					\addplot[ybar, bar width=0.32cm, fill=red!60, draw=red!80]
					coordinates {(3.2, 0.05) (4.0, 0.02)};
					\addplot[dashed, thick, black] coordinates {(2.8, 0) (2.8, 0.55)};
					\addplot[only marks, mark=triangle*, mark size=2pt, blue!70!black]
					coordinates {(1.4, 0)};
					\addplot[very thick, red!70!black] coordinates {(3.5, 0) (3.5, 0.55)};
				\end{axis}
			\end{tikzpicture}
		};

		\node[above=0.05cm of a2, anchor=south] (dist2) {
			\begin{tikzpicture}
				\begin{axis}[
					width=3.5cm, height=2.2cm,
					axis lines=left,
					xtick=\empty, ytick=\empty,
					xmin=0, xmax=5, ymin=0, ymax=0.6,
					clip=false,
					]
					\addplot[ybar, bar width=0.32cm, fill=blue!40, draw=blue!60]
					coordinates {(0.8, 0.28) (1.6, 0.32) (2.4, 0.22)};
					\addplot[ybar, bar width=0.32cm, fill=red!60, draw=red!80]
					coordinates {(3.2, 0.12) (4.0, 0.06)};
					\addplot[dashed, thick, black] coordinates {(2.8, 0) (2.8, 0.55)};
					\addplot[only marks, mark=triangle*, mark size=2pt, blue!70!black]
					coordinates {(1.9, 0)};
					\addplot[thick, red!70!black] coordinates {(3.5, 0) (3.5, 0.55)};
				\end{axis}
			\end{tikzpicture}
		};

		\node[successor, below left=0.5cm and 0.25cm of a1] (b1_1) {$\bar{b}_1^1$};
		\node[successor, below=0.5cm of a1] (b1_2) {$\bar{b}_1^2$};
		\node[successor, below right=0.5cm and 0.25cm of a1] (b1_3) {$\bar{b}_1^3$};

		\draw[arrow] (a1) -- (b1_1);
		\draw[arrow] (a1) -- (b1_2);
		\draw[arrow] (a1) -- (b1_3);

		\node[successor, below left=0.5cm and 0.25cm of a2] (b2_1) {$\bar{b}_1^4$};
		\node[successor, below=0.5cm of a2] (b2_2) {$\bar{b}_1^5$};
		\node[successor, below right=0.5cm and 0.25cm of a2] (b2_3) {$\bar{b}_1^6$};

		\draw[arrow] (a2) -- (b2_1);
		\draw[arrow] (a2) -- (b2_2);
		\draw[arrow] (a2) -- (b2_3);

		\path (b1_2) -- (b2_2) coordinate[midway] (legend_center);
		\node[below=0.3cm of legend_center, anchor=north] {
			\begin{tikzpicture}
				\fill[blue!40] (-2.5,-0.1) rectangle (-2.25,0.1);
				\node[right, font=\scriptsize] at (-2.2,0) {non-tail};
				\fill[red!60] (-1.1,-0.1) rectangle (-0.85,0.1);
				\node[right, font=\scriptsize] at (-0.8,0) {CVaR tail};
				\draw[dashed, thick, black] (0.5,-0.15) -- (0.5,0.15);
				\node[right, font=\scriptsize] at (0.6,0) {VaR};
				\fill[blue!70!black] (1.3,0) -- (1.43,0.15) -- (1.56,0) -- cycle;
				\node[right, font=\scriptsize] at (1.65,0) {$\mathbb{E}$};
				\draw[line width=1.6pt, red!70!black] (2.3,-0.15) -- (2.3,0.15);
				\node[right, font=\scriptsize] at (2.4,0) {CVaR};
			\end{tikzpicture}
		};

	\end{tikzpicture}
	\caption{ICVaR sparse sampling tree with value distributions at action nodes. Blue bars: non-tail region; red bars: upper $\alpha$-tail. Dashed line: VaR threshold; solid red line: CVaR.}
	\label{fig:icvar_sparse_sampling_tree}
\end{figure}

\section{Performance Guarantees}\label{sec:performance_guarantees}
This section establishes finite-time performance guarantees for the ICVaR policy evaluation and sparse sampling algorithms. To the best of our knowledge, such guarantees for ICVaR-based sparse sampling have not been previously established in the literature.

\subsection{ICVaR Policy Evaluation Guarantees}
The following theorem provides probabilistic bounds on the estimation error of Algorithm~\ref{alg:icvar_sparse_sampling_policy_evaluation}.

\begin{theorem}\label{thm:guarantees_for_PB_MDP_vs_algo_v2}
	Let $\delta \in (0,1)$, and particle belief at time $t$ be $\bar{b}_t=\{x_t^i, w_t^i\}_{i=1}^{N_p}$. Define $\Delta R \triangleq R_{\max}-R_{\min}$, $T_{\alpha,t}\triangleq\sum_{k=0}^{T-t-1} \frac{T-t-k}{\alpha^k}$, and $T'_{\alpha,t}\triangleq\sum_{j=0}^{T-t-1} \frac{T-t+1-j}{\alpha^j}$. If $N_b>1$,
	\begin{equation}
		\begin{aligned}
			&P \Bigg( Q_{M_P, t}^\pi(\bar{b}_t, a, \alpha) - \hat{Q}^\pi_{M_P, t}(\bar{b}_t, a, \alpha) \\
			&\leq \gamma \Delta R \cdot T_{\alpha,t}\sqrt{\frac{5\ln(\frac{3(N_b^{T-t}-1)}{\delta(N_b-1)})}{\alpha N_b}} \Bigg)
			\geq 1-\delta.
		\end{aligned}
	\end{equation}
	\begin{equation}
		\begin{aligned}
			&P\Bigg(Q_{M_P, t}^\pi(\bar{b}_t, a, \alpha) - \hat{Q}^\pi_{M_P, t}(\bar{b}_t, a, \alpha) \\
			&\geq -\frac{\gamma \Delta R}{\alpha}\sqrt{\frac{\ln\left(\frac{1 - N_b^{T-t}}{\delta(1 - N_b)}\right)}{2 N_b}} T'_{\alpha,t}\Bigg)
			\geq 1 - \delta.
		\end{aligned}
	\end{equation}
\end{theorem}

\begin{proof}
	The proof is available in the supplemental material.
\end{proof}
\subsection{ICVaR Sparse Sampling Guarantees}
The following theorem extends the policy evaluation guarantees to the sparse sampling setting of Algorithm~\ref{alg:icvar_sparse_sampling}, where the optimal policy is computed via Bellman optimality.

\begin{theorem}\label{thm:guarantees_for_sparse_sampling}
	Let $\delta \in (0,1)$, particle belief at time $t$ be $\bar{b}_t=\{x_t^i, w_t^i\}_{i=1}^{N_p}$, and let $|A|$ denote the number of actions. With $\Delta R$, $T_{\alpha,t}$, and $T'_{\alpha,t}$ as in Theorem~\ref{thm:guarantees_for_PB_MDP_vs_algo_v2}, if $N_b > 1$,
	\begin{equation}
		\begin{aligned}
			&P \Bigg( V_{M_P, t}^*(\bar{b}_t, \alpha) - \hat{V}^*_{M_P, t}(\bar{b}_t, \alpha) \\
			&\leq \gamma \Delta R \cdot T_{\alpha,t}\sqrt{\frac{5\ln\big(\frac{3|A|((|A|N_b)^{T-t}-1)}{\delta(|A|N_b-1)}\big)}{\alpha N_b}} \Bigg) \geq 1-\delta.
		\end{aligned}
	\end{equation}
	\begin{equation}
		\begin{aligned}
			&P\Bigg(V_{M_P, t}^*(\bar{b}_t, \alpha) - \hat{V}^*_{M_P, t}(\bar{b}_t, \alpha) \\
			&\geq -\frac{\gamma \Delta R}{\alpha}\sqrt{\frac{\ln\big(\frac{|A|((|A|N_b)^{T-t}-1)}{\delta(|A|N_b - 1)}\big)}{2 N_b}} T'_{\alpha,t}\Bigg) \geq 1 - \delta.
		\end{aligned}
	\end{equation}
\end{theorem}

\begin{proof}[Proof sketch]
Since $V^*_{M_P,t}(\bar{b}_t, \alpha) = \min_a Q^*_{M_P,t}(\bar{b}_t, a, \alpha)$, bounding $|Q^* - \hat{Q}^*|$ for all actions yields bounds on $|V^* - \hat{V}^*|$. For any action $a$, the Q-function error decomposes as
\begin{align*}
	&Q^*_{M_P,t}(\bar{b}_t, a, \alpha) - \hat{Q}^*_{M_P,t}(\bar{b}_t, a, \alpha) \\
	&= \gamma \Big( \underbrace{CVaR_\alpha[V^*_{M_P,t+1}] - \hat{C}_\alpha[\{V^*_{M_P,t+1}(\bar{b}_{t+1}^i)\}_{i=1}^{N_b}]}_{\text{(I) CVaR sampling error}} \\
	&\quad + \underbrace{\hat{C}_\alpha[\{V^*_{M_P,t+1}(\bar{b}_{t+1}^i)\}] - \hat{C}_\alpha[\{\hat{V}^*_{M_P,t+1}(\bar{b}_{t+1}^i)\}]}_{\text{(II) Propagated estimation error}} \Big).
\end{align*}
Term (I) is bounded using CVaR concentration inequalities for i.i.d.\ samples. Term (II) is bounded via $\text{(II)} \leq \frac{1}{N_b\alpha}\sum_{i=1}^{N_b}(V^*_{M_P,t+1}(\bar{b}_{t+1}^i) - \hat{V}^*_{M_P,t+1}(\bar{b}_{t+1}^i))^+$, a one-sided inequality for empirical CVaR differences. Applying a union bound over all $|A|$ actions and $N_b$ successor beliefs, then recursing over the tree yields the stated bounds. The full proof is available in the supplemental material.
\end{proof}

\subsection{Remarks}
The bounds in Theorem~\ref{thm:guarantees_for_sparse_sampling} differ from those in Theorem~\ref{thm:guarantees_for_PB_MDP_vs_algo_v2} by replacing $N_b$ with $|A|N_b$ in the exponential terms and introducing an additional factor of $|A|$ in the logarithm. This reflects the union bound over all actions required when computing optimal values via Bellman optimality. When $|A| = 1$, the bounds reduce to those of Theorem~\ref{thm:guarantees_for_PB_MDP_vs_algo_v2}.

\textbf{Special cases.} When $\alpha = 1$, both theorems reduce to the standard expectation-based setting. When $N_b = 1$, the PB-MDP collapses to a standard MDP; corresponding bounds are available in the supplemental material.

\textbf{Convergence.} Both bounds are $O(\sqrt{\ln(N_b)/N_b})$, implying convergence to the true value as $N_b \to \infty$. For fixed $\delta$, $T$, and $\alpha$, the estimation error vanishes as $N_b$ increases. 

\section{MCTS Algorithms}
This section introduces MCTS-based algorithms that optimize the ICVaR value function for POMDP planning, extending the expectation-optimizing methods POMCPOW and PFT-DPW \cite{sunberg2018online}. The algorithms introduce a parameter $\alpha \in (0,1]$ that controls the level of risk sensitivity. When $\alpha = 1$, the algorithms perform expectation-based optimization analogous to the original POMCPOW and PFT-DPW algorithms, whereas for $\alpha < 1$, they  incorporate risk considerations into the planning process.

We employ the following notation throughout this section. A history is denoted by $h = (b, a_1, o_1, \dots, a_k, o_k)$, while $ha$ and $hao$ denote the history $h$ followed by an action $a$ and by an action–observation pair $(a,o)$, respectively. The variable $d$ denotes the current search depth, with $d_{\max}$ indicating the maximum allowable depth. The set $C$ represents the children of a node (together with the associated reward in the case of PFT-DPW). The variable $N$ records the number of visits to a node, $M$ records the number of times a history is generated by the model, and the immediate cost is denoted by $Imm$.

Each node maintains a list of associated states, denoted by $B$, with corresponding weights $W$. Furthermore, $V(ha)$ denotes the estimated value of executing action $a$ following history $h$. Unless otherwise specified, all variables $C, N, M, B, W, Imm$, and $V$ are implicitly initialized to $0$ or $\emptyset$.

\subsection{ICVaR Exploration}
\label{subsec:icvar_exploration}

A key challenge in extending MCTS to ICVaR objectives lies in the exploration mechanism. Standard UCB-based exploration relies on Hoeffding's inequality to construct confidence bounds on action-values, ensuring sufficient exploration while concentrating samples on promising actions. However, Hoeffding's inequality assumes that samples contribute linearly to a mean estimate. ICVaR, as a tail-risk measure, weights outcomes non-uniformly---only the worst $\alpha$-fraction of outcomes contribute to the value estimate. This renders standard concentration bounds invalid for ICVaR value estimation.

To address this, we introduce \textbf{ICVaR Progressive Widening}, which leverages ICVaR sparse sampling guarantees ICVaR policy evaluation guarantees (Theorem~\ref{thm:guarantees_for_PB_MDP_vs_algo_v2})  to construct valid confidence bounds on the action-value function during exploration. Specifically, we replace the constants that do not depend on the branching factor $N_b$ by $c \geq 0$, and note that $N_b$ corresponds to the visit count of an action node. To construct a UCB-like bound, we note that $M(ha)\leq M(h)$ and therefore 
\begin{equation}\label{eq:icvar_ucb_bound}
	c\sqrt{\frac{\ln\!\left(\frac{1 - {N(ha)}^{T-t}}{\delta(1 - N(ha))}\right)}{\alpha N(ha)}} \leq c\sqrt{\frac{\ln\!\left(\frac{1 - {N(h)}^{T-t}}{\delta(1 - N(h))}\right)}{\alpha N(ha)}}.
\end{equation}
The left-hand side of \eqref{eq:icvar_ucb_bound} is proportional to the bound in Theorem~\ref{thm:guarantees_for_PB_MDP_vs_algo_v2}, while the right-hand side is analogous to a UCB term, increasing for infrequently selected actions and decreasing as actions become well explored.

\begin{algorithm}[H]
	\caption{Common Procedures}
	\label{alg:common-procedures}
	\begin{algorithmic}[1]
		\STATE \textbf{Constant variables:} $\delta\in (0,1], \alpha\in(0,1]$
		
		\STATE \textbf{Function} \textsc{ICVaRActionProgWiden}($h$)
		\IF{$|C(h)| \le k_a \, N(h)^{\alpha_a}$}
		\STATE $a \gets \textsc{NextAction}(h)$
		\STATE $C(h) \gets C(h) \cup \{ha\}$
		\ENDIF
		\STATE \textbf{return} \textsc{ICvarExploration}($h$)
		\STATE \textbf{End Function}
		
		\STATE
		\STATE \textbf{Function} \textsc{ICvarExploration}($h$)
		\IF{$\exists a \in C(h)$ such that $M(ha) = 0$}
		\STATE \textbf{return} $a$
		\ENDIF
		\STATE $a \gets 
		\arg\min_{a \in C(h)}
		\Bigg[
		V(ha)
		- c\sqrt{\frac{\ln\!\left(\frac{1 - {M(h)}^{T-t}}{\delta(1 - M(h))}\right)}{\alpha M(ha)}}
		\Bigg]$
		\STATE \textbf{return} $a$
		\STATE \textbf{End Function}
		\STATE
		\STATE \textbf{Function} \textsc{GenPF}($\bar{b}, a$)
		\STATE \textbf{Input:} Particle belief set $\bar{b} = \{(x_i, w_i)\}$, action $a$
		\STATE \textbf{Output:} New updated particle belief set $\bar{b'} = \{(x'_i, w'_i)\}$, mean cost $\rho$
		\STATE $x_0 \gets$ sample $x_i$ from $\bar{b}$ w.p. $w_i / \sum_i w_i$
		\STATE $z \gets G(x_0, a)$
		\FOR{$i = 1$ to $N_p$}
		\STATE $x'_i, c_i \gets G(x_i, a)$
		\STATE $w'_i \gets w_i \cdot Z(z|a, x'_i)$
		\ENDFOR
		\STATE $\bar{b'} \gets \{(x'_i, w'_i)\}_{i=1}^{N_x}$
		\STATE $\rho \gets \sum_i w_i c_i / \sum_i w_i$
		\STATE \textbf{return} $\bar{b'}, \rho$
		\STATE \textbf{End Function}
	\end{algorithmic}
\end{algorithm}

\begin{algorithm}[H]
	\caption{ICVaR-PFT-DPW}
	\label{alg:pft-dpw}
	\textbf{Global Variables}: $N, V, Q, C, n_0, d_{\max}, k_o, \alpha_o, \alpha, \gamma$
	\begin{algorithmic}[1]
		\STATE \textbf{Function} \textsc{Plan}($b$)
		\FOR{$i = 1$ to $n_0$}
		\STATE \textsc{Simulate}($b, d_{\max}$)
		\ENDFOR
		\STATE \textbf{return} $\arg\min_a Q(ba)$
		\STATE \textbf{End Function}
		\STATE
		\STATE \textbf{Function} \textsc{Simulate}($b, d$)
		\IF{$d = 0$}
		\STATE \textbf{return} $0$
		\ENDIF
		\STATE \alghl{$a \gets \textsc{ICVaRActionProgWiden}(b)$} \hfill (Alg.~\ref{alg:common-procedures})
		\IF{$|C(ba)| \le k_o \, N(ba)^{\alpha_o}$}
		\STATE $(b', c) \gets \textsc{GenPF}(b, a)$ \hfill (Alg.~\ref{alg:common-procedures})
		\STATE $C(ba) \gets C(ba) \cup \{(b', c)\}$
		\ELSE
		\STATE $(b', c) \gets$ sample uniformly from $C(ba)$
		\ENDIF
		\STATE \alghl{SIMULATE($b'$, $d - 1$)}
		\STATE $N(b) \gets N(b) + 1$
		\STATE $N(ba) \gets N(ba) + 1$
		\STATE \alghl{$Q(ba) \gets c + \hat{C}_\alpha\left(\{V(b')\}_{b'\in C(ba)}\right)$}
		\STATE $V(b) \gets \min_{a \in C(b)} Q(ba)$
		\STATE \textbf{return} None 
		\STATE \textbf{End Function}
	\end{algorithmic}
\end{algorithm}

\subsection{ICVaR-POMCPOW and ICVaR-PFT-DPW}
\label{subsec:icvar_pomcpow_pft_dpw}
Both ICVaR-POMCPOW and ICVaR-PFT-DPW (Algorithm~\ref{alg:pomcpow} and  Algorithm~\ref{alg:pft-dpw}) are MCTS variants that retain the three fundamental steps: selection, expansion, and backpropagation. They employ a \textsc{Plan} function that receives a user-specified belief and returns the selected action. In addition to ICVaR Progressive Widening for action selection, both algorithms utilize Progressive Widening to regulate the branching factor at observation nodes. The widening schedule is chosen in accordance with Theorem~\ref{thm:guarantees_for_PB_MDP_vs_algo_v2} to ensure the approximation error remains bounded. Lines highlighted in {\color{blue}blue} indicate modifications from the original POMCPOW and PFT-DPW algorithms.

A key difference from the original algorithms is the absence of random rollouts. Standard MCTS algorithms use rollouts to generate fast estimates for recursive value updates. However, the CVaR statistic cannot be updated recursively, as it depends on the entire distribution of outcomes rather than individual samples. Instead, both ICVaR algorithms track a full path from the root belief node to the planning depth. Upon reaching a leaf node, a backpropagation phase updates the value estimates of all parent nodes along the trajectory up to the root, computing CVaR over the children of each action node. 

The primary distinction between the two algorithms lies in the \textsc{Simulate} function. ICVaR-POMCPOW (Algorithm~\ref{alg:pomcpow}) simulates state trajectories: at each step, a state is sampled from the current belief and propagated through the POMDP dynamics, with observations sampled according to the progressive widening criterion. ICVaR-PFT-DPW (Algorithm~\ref{alg:pft-dpw}) instead simulates belief trajectories by propagating particle filters through the tree using the $GenPF_{(m)}$ formulation \cite{sunberg2018online}. This design enables the algorithm to operate on problems where the belief update mechanism is abstract and user-defined. 

\begin{algorithm}[H]
	\caption{ICVaR-POMCPOW}
	\label{alg:pomcpow}
	\textbf{Global Variables}: $N, V, B, W, \mathcal{C}, M, Imm, n, d_{\max}, k_o, \alpha_o, \alpha, \gamma$
	\begin{algorithmic}[1]
		\STATE \textbf{Function} \textsc{Plan}($b$)
		\STATE $h \gets (b)$
		\FOR{$i = 1$ to $n$}
		\STATE Sample $s \sim b$
		\STATE \textsc{Simulate}($s, h, d_{\max}$)
		\ENDFOR
		\STATE \textbf{return} $\arg\min_a V(ha)$
		\STATE \textbf{End Function}
		\STATE
		\STATE \textbf{Function} \textsc{Simulate}($s, h, d$)
		\IF{$d = 0$}
		\STATE \textbf{return} None
		\ENDIF
		\STATE \alghl{$a \gets \textsc{ICVaRActionProgWiden}(h)$} \hfill (Alg.~\ref{alg:common-procedures})
		\STATE $(s', o, c) \sim G(s, a)$
		\STATE \alghl{$\text{OldWeightsSum} \gets \sum_{o\in \mathcal{C}(ha)} W(hao)$}
		\IF{$|\mathcal{C}(ha)| \leq k_o N(ha)^{\alpha_o}$}
		\STATE $M(hao) \gets M(hao) + 1$
		\ELSE
		\STATE $o \gets \text{select } o \in \mathcal{C}(ha)$ \text{w.p. $M(hao)/\sum M(hao)$}
		\ENDIF
		\STATE Append $s'$ to $B(hao)$
		\STATE Append $P(o \mid s, a, s')$ to $W(hao)$
		\IF{$o \notin \mathcal{C}(ha)$}
		\STATE $\mathcal{C}(ha) \gets \mathcal{C}(ha) \cup \{o\}$
		\ELSE
		\STATE $s' \gets \text{select } B(hao)[i]$ \text{ w.p. $W(hao)[i]/\sum_j W(hao)[j]$}
		\STATE $c \gets -R(s, a, s')$
		\ENDIF
		\STATE \alghl{SIMULATE(s', hao, d - 1)}
		\STATE $N(ha) \gets N(ha) + 1$
		\STATE $N(h) \gets N(h) + 1$
		\STATE \alghl{$\text{NewWeightsSum} \gets \text{OldWeightsSum}+P(o \mid s, a, s')$}
		\STATE \alghl{$Imm(ha) \gets \frac{Imm(ha)\times \text{OldWeightsSum} + c * P(o \mid s, a, s')}{\text{NewWeightsSum}}$}
		\STATE \alghl{$V(ha) \gets Imm(ha) + \hat{C}_\alpha\left(\{V(hao)\}_{o\in \mathcal{C}(ha)}\right)$} 
		\STATE $V(h) \gets \min_{a \in Ch(h)} V(ha)$
		\STATE \textbf{return} \text{None}
		\STATE \textbf{End Function}
	\end{algorithmic}
\end{algorithm}

\section{Experiments}
\begin{table}[h]
	\centering
	\small
	\setlength{\tabcolsep}{4pt}
	\renewcommand{\arraystretch}{1.2}
	\begin{tabular}{l cc}
		\toprule
		\textbf{Method}
		& \textbf{LaserTag (D,D,C)}
		& \textbf{LightDark (C,C,C)} \\
		\midrule
		POMCPOW
		& $15.06 \pm 0.40$
		& $25.73 \pm 0.96$ \\
		\textbf{ICVaR-POMCPOW}
		& $\mathbf{12.47 \pm 0.46}$
		& $\mathbf{16.72 \pm 0.08}$ \\
		\midrule
		PFT-DPW
		& $26.04 \pm 0.91$
		& $37.68 \pm 1.68$ \\
		\textbf{ICVaR-PFT-DPW}
		& $\mathbf{16.33 \pm 0.61}$
		& $\mathbf{18.52 \pm 0.23}$ \\
		\bottomrule
	\end{tabular}
	\caption{ICVaR performance ($\alpha = 0.1$) across 200 episodes per planner. Labels (D,D,C) denote discrete/continuous state, action, and observation spaces. Lower is better. Confidence intervals at 95\% level.}
	\label{tab:pomdp_results}
\end{table}

We evaluate our methods on two benchmark POMDP environments. LaserTag tests belief-driven pursuit and evasion in a grid world with noisy range sensing, while LightDark emphasizes localization-aware navigation under position-dependent observation noise. Together, these environments assess planning and belief maintenance under qualitatively different partial observability challenges.

We compare ICVaR-POMCPOW and ICVaR-PFT-DPW against their risk-neutral counterparts, POMCPOW and PFT-DPW. The ICVaR planners use $\alpha = 0.1$ and $\delta = 0.05$, with a planning budget of 4 seconds per step and horizon $T = 10$. The exploration constant is set to $c = T(R_{\max} - R_{\min})$. Each planner is evaluated over 200 episodes per environment. ICVaR is estimated using the policy evaluation algorithm (Section~\ref{sec:icvar_sparse_sampling}) with branching factor $N_b = 5$ and horizon $T = 3$.

Table~\ref{tab:pomdp_results} shows that ICVaR-optimizing planners consistently outperform their risk-neutral counterparts. ICVaR-POMCPOW reduces ICVaR by 17\% on LaserTag and 35\% on LightDark compared to POMCPOW. Similarly, ICVaR-PFT-DPW achieves reductions of 37\% and 51\% over PFT-DPW on the respective environments. These improvements demonstrate that explicitly optimizing for tail risk yields policies that perform better under the ICVaR criterion.

\section{Conclusions}
This paper introduced a framework for risk-averse online planning in POMDPs using the Iterated Conditional Value-at-Risk (ICVaR) dynamic risk measure. We developed ICVaR variants of three widely used planning algorithms—Sparse Sampling, POMCPOW, and PFT-DPW—that optimize a risk-sensitive objective rather than the expected return. To our knowledge, this work is the first to propose online planning methods in a risk-averse POMDP setting.

Our theoretical contributions include finite-time performance guarantees for both ICVaR policy evaluation and ICVaR Sparse Sampling. For Sparse Sampling, we derived bounds that explicitly characterize how estimation error scales with the risk parameter $\alpha$, branching factor, and planning horizon. These guarantees enabled a principled exploration strategy for MCTS based on ICVaR-specific concentration inequalities, replacing the standard Hoeffding-based exploration suited to expectation-based objectives.

The risk parameter $\alpha$ allows practitioners to interpolate between risk-neutral planning ($\alpha = 1$) and increasingly conservative strategies ($\alpha < 1$). Experiments on benchmark POMDP domains demonstrate that ICVaR planners achieve lower tail risk than their risk-neutral counterparts, validating the approach for safety-critical applications.

\bibliographystyle{named}
\bibliography{references}

\appendix
\section{Proofs}

\begin{theorem}\label{prf:guarantees_for_PB_MDP_vs_algo_v2}
	Let $\lambda > 0, \delta \in (0,1)$, and particle belief at time t $\bar{b}_t=\{x_t^i, w_t^i\}_{i=1}^{N_p}$. Then for $N_b > 1$,	
	\begin{equation}
		\begin{aligned}
			&P \Big( Q_{M_P, t}^\pi(\bar{b}_t, a, \alpha) - \hat{Q}^\pi_{M_P, t}(\bar{b}_t, a, \alpha) \\
			&\leq \gamma(R_{\max}-R_{\min})\sqrt{\frac{5\ln(\frac{3(N_b^{T-t}-1)}{\delta(N_b-1)})}{\alpha N_b}} \sum_{k=0}^{T-t-1} \frac{T-t-k}{\alpha^k}\Big) \\
			&\geq 1-\delta.
		\end{aligned}
	\end{equation}
	\begin{equation}
		\begin{aligned}
			\begin{split}
				&P\Bigg(Q_{M_P, t}^\pi(\bar{b}_t, a, \alpha) - \hat{Q}^\pi_{M_P, t}(\bar{b}_t, a, \alpha) \\
				&\geq -\frac{\gamma(R_{max} - R_{min})}{\alpha}\sqrt{\frac{\ln\left(\frac{1 - N_b^{T-t}}{\delta(1 - N_b)}\right)}{2 N_b}} \sum_{j=0}^{T-t-1} \frac{T-t+1-j}{\alpha^j}\Bigg) \\
				&\geq 1 - \delta
			\end{split}
		\end{aligned}
	\end{equation}
	For $N_b = 1$,
	\begin{equation}
		\begin{aligned}
			&P \Big( Q_{M_P, t}^\pi(\bar{b}_t, a, \alpha) - \hat{Q}^\pi_{M_P, t}(\bar{b}_t, a, \alpha) \\
			&\leq \gamma(R_{\max}-R_{\min})\sqrt{\frac{5\ln(\frac{3(T-t)}{\delta})}{\alpha}} \sum_{k=0}^{T-t-1} \frac{T-t-k}{\alpha^k}\Big) \\
			&\geq 1-\delta.
		\end{aligned}
	\end{equation}
	\begin{equation}
		\begin{aligned}
			\begin{split}
				&P\Bigg(Q_{M_P, t}^\pi(\bar{b}_t, a, \alpha) - \hat{Q}^\pi_{M_P, t}(\bar{b}_t, a, \alpha) \\
				&\geq -\frac{\gamma(R_{max} - R_{min})}{\alpha}\sqrt{\frac{\ln\left(\frac{T-t}{\delta}\right)}{2}} \sum_{j=0}^{T-t-1} \frac{T-t+1-j}{\alpha^j}\Bigg) \\
				&\geq 1 - \delta
			\end{split}
		\end{aligned}
	\end{equation}
\end{theorem}
\begin{proof}
	The proof organized as follows. We first decompose the error
	$Q_{M_P}^\pi(\bar{b}_t, a, \alpha) - \hat{Q}_{M_P}^\pi(\bar{b}_t, a, \alpha)$
	into two components: a CVaR estimation error and a value function estimation error. We then establish, by induction, recursive probabilistic bounds on this difference. Finally, we derive explicit error bounds by unrolling and analyzing these recursive relations.
	
	We prove this theorem recursively, with a base case of $t=T$. \\
	\begin{equation}
		\begin{aligned}
			&Q_{M_P}^\pi(\bar{b}_t, a, \alpha) - \hat{Q}^\pi_t(\bar{b}_t, a, \alpha) = c(\bar{b}_t, a) - c(\bar{b}_t, a) \\
			&+ \gamma(CVaR_\alpha(V_{M_P,t+1}^\pi(\bar{b}_{t+1}, \alpha)|b_t,a) \\
			&-  \widehat{CVaR}_\alpha(\{\hat{V}_{M_P,t+1}^\pi(\bar{b}_{t+1}^{i})\}_{j=1}^{N_b})) \\
			&=\gamma (CVaR_\alpha(V_{M_P, t+1}^\pi(\bar{b}_{t+1}, \alpha)|b_t,a) \\
			&-  \widehat{CVaR}_\alpha(\{\hat{V}_{M_P, t+1}^\pi(\bar{b}_{t+1}^{i})\}_{j=1}^{N_b}))\\
			&\triangleq B.
		\end{aligned}
	\end{equation}
	We represent the difference between the sparse sampling estimator ad the theoretical action-value function into the sum of the CVaR estimation error and the value function approximation error.
	\begin{equation}\label{eq:B_inequality}
		\begin{aligned}
			&\frac{1}{\gamma} B\leq \underbrace{CVaR_\alpha(V_{M_P, t+1}^\pi(\bar{b}_{t+1}, \alpha)|\bar{b}_t,a)
			}_{\text{(1) CVaR estimation error}}\\
			&\underbrace{-\widehat{CVaR}_\alpha(\{V_{M_P}^\pi(\bar{b}_{t+1}^{i})\}_{j=1}^{N_b})}_{\text{(1) CVaR estimation error}}\\
			&\underbrace{+\widehat{CVaR}_\alpha(\{V_{M_P}^\pi(\bar{b}_{t+1}^{i})\}_{j=1}^{N_b}) -  \widehat{CVaR}_\alpha(\{\hat{V}_{M_P}^\pi(\bar{b}_{t+1}^{i})\}_{j=1}^{N_b})}_{\text{(2) Value function approximation error}}
		\end{aligned}
	\end{equation}
	
	\textbf{Upper bound proof:} We want to prove that the following equation holds for $t\in \{1,\dots,T\}$
	\begin{equation}
		P(Q_{M_P, t}^\pi(\bar{b}_t, a, \alpha) - \hat{Q}^\pi_{M_P, t}(\bar{b}_t, a, \alpha) \leq \gamma\theta_t)\geq 1-\eta_t,
	\end{equation}
	where
	\begin{equation}
		\eta_t=\delta+N_b\eta_{t+1}, \quad \eta_T=0,
	\end{equation}
	\begin{equation}
		\theta_t=(T-t+1)(R_{max}-R_{min}) \sqrt{\frac{5ln(3/\delta)}{\alpha N_b}} + \frac{\theta_{t+1}}{\alpha},
	\end{equation}
	and $\theta_T=0$.
	For the induction base case, note that 
	\begin{equation}
		Q_{M_P, T}^\pi(\bar{b}_T, a, \alpha) - \hat{Q}^\pi_T(\bar{b}_t, a, \alpha)=c(\bar{b}_T,a)-c(\bar{b}_T,a)=0
	\end{equation}
	and therefore the base case holds. We assume the probabilistic bounds hold for $t+1$, and prove for t.
	
	\textbf{(1) CVaR estimation error:} 
	Conditioned on $\bar{b}_t$, the quantity $V_{M_P, t+1}^\pi(\bar{b}_{t+1}, \alpha)$ is a random variable. The sources of randomness are the state at time $t$, $x_t$, the subsequent state $x_{t+1}$, and the observation at time $t+1$, $z_{t+1}$. Given $\bar{b}_t$, $x_t$, $x_{t+1}$, and $z_{t+1}$, the updated belief $\bar{b}_{t+1}$ is deterministic, and therefore all randomness in $V_{M_P, t+1}^\pi(\bar{b}_{t+1}, \alpha)$ is induced through $\bar{b}_{t+1}$.
	
	Moreover, the collection $\{V_{M_P, t+1}^\pi(\bar{b}_{t+1}^{i}, \alpha)\}_{j=1}^{N_b}$ forms an i.i.d. sample from the distribution of $V_{M_P, t+1}^\pi(\bar{b}_{t+1}, \alpha)|\bar{b}_t$, with support bounded by $(R_{\max}-R_{\min})(T-t+1)$. Hence, by Theorem~\ref{thm:brown_bounds},
	\begin{equation}
		\begin{aligned}
			&P \Big(CVaR_\alpha(V_{M_P, t+1}^\pi(\bar{b}_{t+1}, \alpha)|\bar{b}_t,a) \\ &-\widehat{CVaR}_\alpha(\{V_{M_P, t+1}^\pi(\bar{b}_{t+1}^{i})\}_{i=1}^{N_b}) \\
			&\leq (T-t+1)(R_{max}-R_{min}) \sqrt{\frac{5ln(3/\delta)}{\alpha N_b}} \Big)
			\geq 1-\delta,
		\end{aligned}
	\end{equation}
	
	\textbf{(2) Value function approximation error:} For this proof we use the following one-sided inequality between CVaR estimators,
	\begin{equation}
		\begin{aligned}
			&\widehat{CVaR}_\alpha(\{V_{M_P, t+1}^\pi(\bar{b}_{t+1}^{i})\}_{i=1}^{N_b}) \\
			&-  \widehat{CVaR}_\alpha(\{\hat{V}_{M_P, t+1}^\pi(\bar{b}_{t+1}^{i})\}_{i=1}^{N_b}) \\
			&\leq \frac{1}{N_b \alpha}\sum_{i=1}^{N_b} (V_{M_P, t+1}^\pi(\bar{b}_{t+1}^{i}) - \hat{V}_{M_P , t+1}^\pi(\bar{b}_{t+1}^{i}))^+.
		\end{aligned}
	\end{equation}
	Although this bound is likely known in the literature, we were unable to identify an explicit reference. For completeness, we therefore provide a self-contained proof in Lemma~\ref{thm:cvar_subtraction_bounds}.
	
	Now we get a lower bound for the probability we want to bound.
	\begin{equation}
		\begin{aligned}
			&P(\widehat{CVaR}_\alpha(\{V_{M_P, t+1}^\pi(\bar{b}_{t+1}^{i})\}_{i=1}^{N_b}) \\
			&- \widehat{CVaR}_\alpha(\{\hat{V}_{M_P, t+1}^\pi(\bar{b}_{t+1}^{i})\}_{i=1}^{N_b})
			\leq \frac{\theta_{t+1}}{\alpha})\\
			&\geq P(\frac{1}{N_b \alpha}\sum_{i=1}^{N_b} (V_{M_P, t+1}^\pi(\bar{b}_{t+1}^{i}) - \hat{V}_{M_P, t+1}^\pi(\bar{b}_{t+1}^{i}))^+ \leq \frac{\theta_{t+1}}{\alpha})
		\end{aligned}
	\end{equation}
	Denote the event that the induction assumption holds for all belief samples by 
	\begin{equation}
		E=\cap_{j=1}^{N_b} \{V_{M_P, t+1}^\pi(\bar{b}_{t+1}^{i}) - \hat{V}_{M_P, t+1}^\pi(\bar{b}_{t+1}^{i}) \leq \gamma\theta_{t+1}\}.
	\end{equation}
	From the union bound we get,
	\begin{equation}
		\begin{aligned}
			&P(E) \\
			&=1-P(\cup_{j=1}^{N_b} \{V_{M_P, t+1}^\pi(\bar{b}_{t+1}^{i}) - \hat{V}_{M_P, t+1}^\pi(\bar{b}_{t+1}^{i}) > \gamma\theta_{t+1}\}) \\
			&\geq 1-\sum_{j=1}^{N_b} P(V_{M_P, t+1}^\pi(\bar{b}_{t+1}^{i}) - \hat{V}_{M_P, t+1}^\pi(\bar{b}_{t+1}^{i}) > \gamma\theta_{t+1}) \\
			&\geq 1-N_b\eta_{t+1}.
		\end{aligned}
	\end{equation}
	where the last equality holds from the induction assumption. By conditioning on $E$ we get
	\begin{equation}
		\begin{aligned}
			&P(\frac{1}{N_b \alpha}\sum_{j=1}^{N_b} (V_{M_P, t+1}^\pi(\bar{b}_{t+1}^{i}) - \hat{V}_{M_P, t+1}^\pi(\bar{b}_{t+1}^{i}))^+ \leq \frac{\theta_{t+1}}{\alpha}) \\
			&= \underbrace{P(\frac{\gamma}{N_b \alpha}\sum_{j=1}^{N_b} (V_{M_P, t+1}^\pi(\bar{b}_{t+1}^{i}) - \hat{V}_{M_P, t+1}^\pi(\bar{b}_{t+1}^{i}))^+ \leq \gamma\frac{\theta_{t+1}}{\alpha}|E)}_{=1} \\
			&\times P(E) \\
			&+ P(\frac{1}{N_b \alpha}\sum_{j=1}^{N_b} (V_{M_P, t+1}^\pi(\bar{b}_{t+1}^{i}) - \hat{V}_{M_P, t+1}^\pi(\bar{b}_{t+1}^{i}))^+ \leq \frac{\theta_{t+1}}{\alpha}|E^c) \\
			&\times P(E^c) \\
			&\geq P(E)
			\geq 1-N_b\eta_{t+1},
		\end{aligned}
	\end{equation}
	where the last inequality holds from the induction assumption. Now we get a bound for the value function approximation error
	\begin{equation}
		\begin{aligned}
			&P \Big( \widehat{CVaR}_\alpha(\{V_{M_P, t+1}^\pi(\bar{b}_{t+1}^{i})\}_{i=1}^{N_b}) \\
			&-  \widehat{CVaR}_\alpha(\{\hat{V}_{M_P, t+1}^\pi(\bar{b}_{t+1}^{i})\}_{j=1}^{N_b}) \leq \frac{\theta_{t+1}}{\alpha} \Big)\geq 1-N_b\eta_{t+1}.
		\end{aligned}
	\end{equation}
	
	\textbf{Combining the bounds:} By combining the bounds we get 
	\begin{equation}
		\begin{aligned}
			&P(Q_{M_P, t}^\pi(\bar{b}_t, a, \alpha) - \hat{Q}^\pi_{M_P, t}(\bar{b}_t, a, \alpha) \leq \gamma\theta_t) \\
			&\geq P \Big( \gamma CVaR_\alpha(V_{M_P, t+1}^\pi(\bar{b}_{t+1}, \alpha)|\bar{b}_t,a) \\ &-\gamma\widehat{CVaR}_\alpha(\{V_{M_P, t+1}^\pi(\bar{b}_{t+1}^{I_i})\}_{i=1}^{N_b} \\
			&\leq \gamma(T-t+1)(R_{max}-R_{min}) \sqrt{\frac{5ln(3/\delta)}{\alpha N_b}} \Big) \\
			&+P \Big( \gamma\widehat{CVaR}_\alpha(\{V_{M_P, t+1}^\pi(\bar{b}_{t+1}^{I_i})\}_{i=1}^{N_b}) \\
			&-  \gamma\widehat{CVaR}_\alpha(\{\hat{V}_{M_P, t+1}^\pi(\bar{b}_{t+1}^{I_i})\}_{j=1}^{N_b}) 
			\leq \gamma\frac{\theta_{t+1}}{\alpha} \Big) - 1 \\
			&\geq (1-\delta) + (1-N_b\eta_{t+1}) - 1 = 1-\delta-N_b\eta_{t+1} \\
			&=1-\eta_t
		\end{aligned}
	\end{equation}
	and this completes the induction proof. Denote 
	\begin{equation}
		\epsilon_t=\gamma(T-t+1)(R_{max}-R_{min}) \sqrt{\frac{5ln(3/\delta)}{\alpha N_b}}
	\end{equation}
	By expanding the recursive relation we just proved we get
	\begin{equation}
		\begin{aligned}
			\eta_t=\delta+N_b\eta_{t+1}=\delta \sum_{k=0}^{T-t-1} N_b^k=\delta \frac{N_b^{T-t}-1}{N_b-1},
		\end{aligned}
	\end{equation}	
	The general form of the bound is 
	\begin{equation}
		\begin{aligned}
			&\gamma \theta_{t} = \gamma \sum_{k=0}^{T-t-1} \frac{\epsilon_{t+k}}{\alpha^k} \\
			&= \gamma \sum_{k=0}^{T-t-1} \frac{(T-t-k)(R_{\max}-R_{\min})\sqrt{\frac{5\ln(3/\delta)}{\alpha N_b}}}{\alpha^k} \\
			&= \gamma(R_{\max}-R_{\min})\sqrt{\frac{5\ln(3/\delta)}{\alpha N_b}} \sum_{k=0}^{T-t-1} \frac{T-t-k}{\alpha^k}.
		\end{aligned}
	\end{equation}
	Now we get a probabilistic bound without the recursive relations.
	\begin{equation}
		\begin{aligned}
			&P \Big( Q_{M_P, t}^\pi(\bar{b}_t, a, \alpha) - \hat{Q}^\pi_{M_P, t}(\bar{b}_t, a, \alpha) \\
			&\leq \gamma(R_{\max}-R_{\min})\sqrt{\frac{5\ln(3/\delta)}{\alpha N_b}} \sum_{k=0}^{T-t-1} \frac{T-t-k}{\alpha^k}\Big) \\
			&\geq 1-\delta \frac{N_b^{T-t}-1}{N_b-1}.
		\end{aligned}
	\end{equation}
	Our goal is to be able to construct informative confidence intervals given confidence level $\delta$, and therefore we use the following equivalent representation of the bound.
	\begin{equation}
		\begin{aligned}
			&P \Big( Q_{M_P, t}^\pi(\bar{b}_t, a, \alpha) - \hat{Q}^\pi_{M_P, t}(\bar{b}_t, a, \alpha) \\
			&\leq \gamma(R_{\max}-R_{\min})\sqrt{\frac{5\ln(\frac{3(N_b^{T-t}-1)}{\delta(N_b-1)})}{\alpha N_b}} \sum_{k=0}^{T-t-1} \frac{T-t-k}{\alpha^k}\Big) \\
			&\geq 1-\delta.
		\end{aligned}
	\end{equation}

	\textbf{Lower bound proof:} this proof is analogous to the upper bound proof in its structure. We want to prove that the following equation holds for $t\in \{1,\dots,T\}$
	\begin{equation}
		P(Q_{M_P, t}^\pi(\bar{b}_t, a, \alpha) - \hat{Q}^\pi_{M_P, t}(\bar{b}_t, a, \alpha) \geq  \lambda_t)\geq 1-\psi_t
	\end{equation}
	where
	\begin{equation}
		\lambda_t=-\frac{\gamma(T-t+1)(R_{max}-R_{min})}{\alpha} \sqrt{\frac{ln(1/\delta)}{2 N_b}} + \frac{\lambda_{t+1}}{\alpha}
	\end{equation}
	\begin{equation}
		\lambda_T=0, \psi_t=\delta + N_b \psi_{t+1}, \psi_T=0.
	\end{equation}
	For the induction base case, note that 
	\begin{equation}
		Q_{M_P, T}^\pi(\bar{b}_T, a, \alpha) - \hat{Q}^\pi_T(\bar{b}_t, a, \alpha)=c(\bar{b}_T,a)-c(\bar{b}_T,a)=0
	\end{equation}
	and therefore the base case holds. We assume the probabilistic bound holds for $t+1$, and prove for t.
	
	\textbf{(1) CVaR estimation error:} for the same considerations in the upper bound case, we use Theorem \ref{thm:brown_bounds} to provide a lower bound.
	\begin{equation}
		\begin{aligned}
			&P \Big(\gamma CVaR_\alpha(V_{M_P, t+1}^\pi(\bar{b}_{t+1}, \alpha)|\bar{b}_t,a) \\ &-\gamma \widehat{CVaR}_\alpha(\{V_{M_P}^\pi(\bar{b}_{t+1}^{I_i}, \alpha)\}_{i=1}^{N_b}) \\
			&\geq -\frac{\gamma(T-t+1)(R_{max}-R_{min})}{\alpha} \sqrt{\frac{ln(1/\delta)}{2 N_b}} \Big)
			\geq 1-\delta.
		\end{aligned}
	\end{equation}
	
	\textbf{(2) Value function approximation error:}
	\begin{equation}
		\begin{aligned}
			&\gamma\widehat{CVaR}_\alpha(\{V_{M_P, t+1}^\pi(\bar{b}_{t+1}^{i}, \alpha)\}_{i=1}^{N_b}) \\
			&-  \gamma\widehat{CVaR}_\alpha(\{\hat{V}_{M_P, t+1}^\pi(\bar{b}_{t+1}^{i}, \alpha)\}_{i=1}^{N_b}) \\
			&\geq -\frac{\gamma}{N_b \alpha}\sum_{i=1}^{N_b} (\hat{V}_{M_P , t+1}^\pi(\bar{b}_{t+1}^{i}, \alpha) - V_{M_P, t+1}^\pi(\bar{b}_{t+1}^{i}, \alpha))^+.
		\end{aligned}
	\end{equation}
	Now we get a lower bound for the probability we want to bound.
	\begin{equation}
		\begin{aligned}
			&P(\gamma \widehat{CVaR}_\alpha(\{V_{M_P, t+1}^\pi(\bar{b}_{t+1}^{i}, \alpha)\}_{i=1}^{N_b}) \\
			&- \gamma \widehat{CVaR}_\alpha(\{\hat{V}_{M_P, t+1}^\pi(\bar{b}_{t+1}^{i}, \alpha)\}_{i=1}^{N_b})
			\geq \frac{\lambda_{t+1}}{\alpha})\\
			&\geq P(-\frac{\gamma}{N_b}\sum_{i=1}^{N_b} (\hat{V}_{M_P, t+1}^\pi(\bar{b}_{t+1}^{i}, \alpha) - V_{M_P, t+1}^\pi(\bar{b}_{t+1}^{i}, \alpha))^+ \geq \lambda_{t+1})
		\end{aligned}
	\end{equation}
	Denote the event that the induction assumption holds for all belief samples by 
	\begin{equation}
		E=\cap_{i=1}^{N_b} \{V_{M_P, t+1}^\pi(\bar{b}_{t+1}^{i}, \alpha) - \hat{V}_{M_P, t+1}^\pi(\bar{b}_{t+1}^{i}, \alpha) \geq \lambda_{t+1}\}.
	\end{equation}	
	From the union bound we get,
	\begin{equation}\label{eq:P_E_lower_bound}
		\begin{aligned}
			&P(E) \\
			&=1-P(\cup_{i=1}^{N_b} \{V_{M_P, t+1}^\pi(\bar{b}_{t+1}^{i}, \alpha) - \hat{V}_{M_P, t+1}^\pi(\bar{b}_{t+1}^{i}, \alpha) < \lambda_{t+1}\}) \\
			&\geq 1-\sum_{i=1}^{N_b} P(V_{M_P, t+1}^\pi(\bar{b}_{t+1}^{i}, \alpha) - \hat{V}_{M_P, t+1}^\pi(\bar{b}_{t+1}^{i}, \alpha) < \lambda_{t+1}) \\
			&\geq 1-N_b \psi_{t+1}.
		\end{aligned}
	\end{equation}
	where the last equality holds from the induction assumption. By conditioning on $E$ we get
	\begin{equation}
		\begin{aligned}
			&P(\frac{\gamma}{N_b \alpha}\sum_{i=1}^{N_b} (\hat{V}_{M_P, t+1}^\pi(\bar{b}_{t+1}^{i}, \alpha) - V_{M_P, t+1}^\pi(\bar{b}_{t+1}^{i}, \alpha))^+ \leq \frac{-\lambda_{t+1}}{\alpha}) \\
			&= \underbrace{P(\frac{\gamma}{N_b}\sum_{i=1}^{N_b} (\hat{V}_{M_P, t+1}^\pi(\bar{b}_{t+1}^{i}, \alpha) - V_{M_P, t+1}^\pi(\bar{b}_{t+1}^{i}, \alpha))^+ \leq -\lambda_{t+1}|E)}_{=1} \\
			&\times P(E) \\
			&+ P(\frac{1}{N_b}\sum_{i=1}^{N_b} (\hat{V}_{M_P, t+1}^\pi(\bar{b}_{t+1}^{i}, \alpha) - V_{M_P, t+1}^\pi(\bar{b}_{t+1}^{i}, \alpha))^+ \leq -\lambda_{t+1}|E^c) \\
			&\times P(E^c) \\
			&\geq P(E)
			\geq 1-N_b\psi_{t+1},
		\end{aligned}
	\end{equation}
	where the last inequality holds from \eqref{eq:P_E_lower_bound}. Now we get a bound for the value function approximation error
	\begin{equation}
		\begin{aligned}
			&P \Big( \gamma \widehat{CVaR}_\alpha(\{V_{M_P, t+1}^\pi(\bar{b}_{t+1}^{i}, \alpha)\}_{i=1}^{N_b}) \\
			&-  \gamma \widehat{CVaR}_\alpha(\{\hat{V}_{M_P, t+1}^\pi(\bar{b}_{t+1}^{i}, \alpha)\}_{i=1}^{N_b}) \geq \frac{\lambda_{t+1}}{\alpha} \Big)\geq 1-N_b \psi_{t+1}.
		\end{aligned}
	\end{equation}
	
	\textbf{Combining the bounds:} By combining the bounds we get 
	\begin{equation}
		\begin{aligned}
			&P(Q_{M_P, t}^\pi(\bar{b}_t, a, \alpha) - \hat{Q}^\pi_{M_P, t}(\bar{b}_t, a, \alpha) \geq \lambda_t) \\
			&\geq P \Big( \gamma CVaR_\alpha(V_{M_P, t+1}^\pi(\bar{b}_{t+1}, \alpha)|\bar{b}_t,a) \\ &-\gamma\widehat{CVaR}_\alpha(\{V_{M_P, t+1}^\pi(\bar{b}_{t+1}^{I_i}, \alpha)\}_{i=1}^{N_b} \\
			&\geq -\frac{\gamma(T-t+1)(R_{max}-R_{min})}{\alpha} \sqrt{\frac{ln(1/\delta)}{2 N_b}} \Big) \\
			&+P \Big( \gamma\widehat{CVaR}_\alpha(\{V_{M_P, t+1}^\pi(\bar{b}_{t+1}^{I_i}, \alpha)\}_{i=1}^{N_b}) \\
			&-  \gamma\widehat{CVaR}_\alpha(\{\hat{V}_{M_P, t+1}^\pi(\bar{b}_{t+1}^{I_i}, \alpha)\}_{j=1}^{N_b}) 
			\geq \frac{\lambda_{t+1}}{\alpha} \Big) - 1 \\
			&\geq (1-\delta) + (1-N_b\psi_{t+1}) - 1 = 1-\delta-N_b\psi_{t+1} \\
			&=1-\psi_t
		\end{aligned}
	\end{equation}
	and this completes the induction proof.
	
	\textbf{Derivation of Closed-Form Solution for $\lambda_t$:}	
	define a constant to simplify notation:
	\begin{equation}
		C = -\frac{\gamma(R_{max} - R_{min})}{\alpha}\sqrt{\frac{\ln(1/\delta)}{2 N_b}}
	\end{equation}
	
	The recursion becomes:
	\begin{equation}
		\lambda_t = C(T-t+1) + \frac{\lambda_{t+1}}{\alpha}
	\end{equation}
	
	\vspace{1em}
	
	Unroll the recursion from $t$ to $T$:
	\begin{align}
		\lambda_t &= C(T-t+1) + \frac{\lambda_{t+1}}{\alpha} \\[10pt]
		&= C(T-t+1) + \frac{1}{\alpha}\left[C(T-t) + \frac{\lambda_{t+2}}{\alpha}\right] \\[10pt]
		&= C(T-t+1) + \frac{C(T-t)}{\alpha} + \frac{\lambda_{t+2}}{\alpha^2} \\[10pt]
		&= C(T-t+1) + \frac{C(T-t)}{\alpha} + \frac{C(T-t-1)}{\alpha^2} + \frac{\lambda_{t+3}}{\alpha^3} \\[10pt]
		&\vdots \nonumber \\[10pt]
		&= \sum_{j=0}^{T-t-1} \frac{C(T-t+1-j)}{\alpha^j} + \frac{\lambda_T}{\alpha^{T-t}} \\[10pt]
		&= C \sum_{j=0}^{T-t-1} \frac{T-t+1-j}{\alpha^j} \quad \text{(since } \lambda_T = 0\text{)}
	\end{align}
	Now we get a fully expanded formula:
	\begin{equation}
		\boxed{\lambda_t = -\frac{\gamma(R_{max} - R_{min})}{\alpha}\sqrt{\frac{\ln(1/\delta)}{2 N_b}} \sum_{j=0}^{T-t-1} \frac{T-t+1-j}{\alpha^j}}.
	\end{equation}
	
	Now we get the bound

	\textbf{Derivation of Closed-Form Solution for $\psi_t$}
	
	\vspace{1em}
	
	We start with the recursive formula:
	\begin{equation}
		\psi_t = \delta + N_b \psi_{t+1}
	\end{equation}
	with boundary condition $\psi_T = 0$.
	
	\vspace{1em}
	
	Unroll the recursion from $t$ to $T$:
	\begin{align}
		\psi_t &= \delta + N_b \psi_{t+1} \\[10pt]
		&= \delta + N_b\left(\delta + N_b \psi_{t+2}\right) \\[10pt]
		&= \delta + N_b \delta + N_b^2 \psi_{t+2} \\[10pt]
		&= \delta + N_b \delta + N_b^2 \delta + N_b^3 \psi_{t+3} \\[10pt]
		&\vdots \nonumber \\[10pt]
		&= \delta \sum_{j=0}^{T-t-1} N_b^j + N_b^{T-t} \psi_T
	\end{align}
	
	Since $\psi_T = 0$:
	\begin{equation}
		\psi_t = \delta \sum_{j=0}^{T-t-1} N_b^j
	\end{equation}
	
	\vspace{1em}
	
	Using the geometric series formula $\sum_{j=0}^{n-1} r^j = \frac{1-r^n}{1-r}$ for $r \neq 1$:
	\begin{equation}
		\boxed{\psi_t = \delta \cdot \frac{1 - N_b^{T-t}}{1 - N_b}}
	\end{equation}
	
	\vspace{1em}
	
	For the special case $N_b = 1$:
	\begin{equation}
		\boxed{\psi_t = \delta(T-t)}
	\end{equation}
	
	\textbf{General Bound Form:}
	\begin{equation}
		\begin{aligned}
			&P\Big( Q_{M_P, t}^\pi(\bar{b}_t, a, \alpha) - \hat{Q}^\pi_{M_P, t}(\bar{b}_t, a, \alpha) \\
			&\geq -\frac{\gamma(R_{max} - R_{min})}{\alpha}\sqrt{\frac{\ln(1/\delta)}{2 N_b}} \sum_{j=0}^{T-t-1} \frac{T-t+1-j}{\alpha^j} \Big) \\
			&\geq 1-\delta \sum_{j=0}^{T-t-1} N_b^j
		\end{aligned}
	\end{equation}
	
	\textbf{Probability Bound with Closed-Form Expressions}
	
	\vspace{1em}
	
	For $N_b \neq 1$:
	\begin{equation}
		\begin{aligned}
			\begin{split}
				&P\Bigg(Q_{M_P, t}^\pi(\bar{b}_t, a, \alpha) - \hat{Q}^\pi_{M_P, t}(\bar{b}_t, a, \alpha) \\
				&\geq -\frac{\gamma(R_{max} - R_{min})}{\alpha}\sqrt{\frac{\ln\left(\frac{1 - N_b^{T-t}}{\delta(1 - N_b)}\right)}{2 N_b}} \sum_{j=0}^{T-t-1} \frac{T-t+1-j}{\alpha^j}\Bigg) \\
				&\geq 1 - \delta
			\end{split}
		\end{aligned}
	\end{equation}
	
	\vspace{1em}
	
	For $N_b = 1$:
	\begin{equation}
		\begin{aligned}
			\begin{split}
				&P\Bigg(Q_{M_P, t}^\pi(\bar{b}_t, a, \alpha) - \hat{Q}^\pi_{M_P, t}(\bar{b}_t, a, \alpha) \\
				&\geq -\frac{\gamma(R_{max} - R_{min})}{\alpha}\sqrt{\frac{\ln\left(\frac{T-t}{\delta}\right)}{2}} \sum_{j=0}^{T-t-1} \frac{T-t+1-j}{\alpha^j}\Bigg) \\[10pt]
				&\geq 1 - \delta
			\end{split}
		\end{aligned}
	\end{equation}

\end{proof}

\begin{theorem}\label{prf:guarantees_for_sparse_sampling}
	Let $\delta \in (0,1)$, particle belief at time $t$ given by $\bar{b}_t=\{x_t^i, w_t^i\}_{i=1}^{N_p}$, and let $|A|$ denote the number of actions. Then for $N_b > 1$,
	\begin{equation}
		\begin{aligned}
			&P \Big( V_{M_P, t}^*(\bar{b}_t, \alpha) - \hat{V}^*_{M_P, t}(\bar{b}_t, \alpha) \\
			&\leq \gamma(R_{\max}-R_{\min})\sqrt{\frac{5\ln\big(\frac{3|A|((|A|N_b)^{T-t}-1)}{\delta(|A|N_b-1)}\big)}{\alpha N_b}} \\
			&\quad \times \sum_{k=0}^{T-t-1} \frac{T-t-k}{\alpha^k}\Big) \geq 1-\delta.
		\end{aligned}
	\end{equation}
	\begin{equation}
		\begin{aligned}
			&P\Big(V_{M_P, t}^*(\bar{b}_t, \alpha) - \hat{V}^*_{M_P, t}(\bar{b}_t, \alpha) \\
			&\geq -\frac{\gamma(R_{\max} - R_{\min})}{\alpha}\sqrt{\frac{\ln\big(\frac{|A|((|A|N_b)^{T-t}-1)}{\delta(|A|N_b - 1)}\big)}{2 N_b}} \\
			&\quad \times \sum_{j=0}^{T-t-1} \frac{T-t+1-j}{\alpha^j}\Big) \geq 1 - \delta.
		\end{aligned}
	\end{equation}
	For $N_b = 1$,
	\begin{equation}
		\begin{aligned}
			&P \Big( V_{M_P, t}^*(\bar{b}_t, \alpha) - \hat{V}^*_{M_P, t}(\bar{b}_t, \alpha) \\
			&\leq \gamma(R_{\max}-R_{\min})\sqrt{\frac{5\ln\big(\frac{3|A|^{T-t}(T-t)}{\delta}\big)}{\alpha}} \\
			&\quad \times \sum_{k=0}^{T-t-1} \frac{T-t-k}{\alpha^k}\Big) \geq 1-\delta.
		\end{aligned}
	\end{equation}
	\begin{equation}
		\begin{aligned}
			&P\Big(V_{M_P, t}^*(\bar{b}_t, \alpha) - \hat{V}^*_{M_P, t}(\bar{b}_t, \alpha) \\
			&\geq -\frac{\gamma(R_{\max} - R_{\min})}{\alpha}\sqrt{\frac{\ln\big(\frac{|A|^{T-t}(T-t)}{\delta}\big)}{2}} \\
			&\quad \times \sum_{j=0}^{T-t-1} \frac{T-t+1-j}{\alpha^j}\Big) \geq 1 - \delta.
		\end{aligned}
	\end{equation}
\end{theorem}
\begin{proof}
	This proof extends Theorem~\ref{prf:guarantees_for_PB_MDP_vs_algo_v2} from policy evaluation to planning. The key difference is that computing $V^*$ requires optimizing over all actions, introducing additional union bounds over $|A|$.
	
	Recall that
	\begin{equation}
		V_{M_P}^*(\bar{b}_t, \alpha) = \min_{a \in A} Q_{M_P}^*(\bar{b}_t, a, \alpha),
	\end{equation}
	\begin{equation}
		Q_{M_P}^*(\bar{b}_t, a, \alpha) = c(\bar{b}_t, a) + \gamma \, CVaR_\alpha(V_{M_P}^*(\bar{b}_{t+1}, \alpha) | \bar{b}_t, a).
	\end{equation}
	
	\textbf{Relating value function bounds to action-value bounds:}
	Let $a^* = \arg\min_{a} Q_{M_P}^*(\bar{b}_t, a, \alpha)$ and $\hat{a} = \arg\min_{a} \hat{Q}_{M_P}^*(\bar{b}_t, a, \alpha)$. For the upper bound:
	\begin{equation}
		\begin{aligned}
			&V_{M_P}^*(\bar{b}_t, \alpha) - \hat{V}_{M_P}^*(\bar{b}_t, \alpha) \\
			&= Q_{M_P}^*(\bar{b}_t, a^*, \alpha) - \hat{Q}_{M_P}^*(\bar{b}_t, \hat{a}, \alpha) \\
			&\leq Q_{M_P}^*(\bar{b}_t, \hat{a}, \alpha) - \hat{Q}_{M_P}^*(\bar{b}_t, \hat{a}, \alpha),
		\end{aligned}
	\end{equation}
	since $Q_{M_P}^*(\bar{b}_t, a^*, \alpha) \leq Q_{M_P}^*(\bar{b}_t, \hat{a}, \alpha)$ by definition of $a^*$. For the lower bound:
	\begin{equation}
		\begin{aligned}
			&V_{M_P}^*(\bar{b}_t, \alpha) - \hat{V}_{M_P}^*(\bar{b}_t, \alpha) \\
			&\geq Q_{M_P}^*(\bar{b}_t, a^*, \alpha) - \hat{Q}_{M_P}^*(\bar{b}_t, a^*, \alpha),
		\end{aligned}
	\end{equation}
	since $\hat{Q}_{M_P}^*(\bar{b}_t, \hat{a}, \alpha) \leq \hat{Q}_{M_P}^*(\bar{b}_t, a^*, \alpha)$ by definition of $\hat{a}$. Since $a^*$ and $\hat{a}$ are unknown, we require bounds on $Q^* - \hat{Q}^*$ for all actions, necessitating a union bound over $|A|$.
	
	\textbf{Upper bound proof:} We prove by induction that for $t \in \{1, \dots, T\}$:
	\begin{equation}
		P(V_{M_P,t}^*(\bar{b}_t, \alpha) - \hat{V}_{M_P,t}^*(\bar{b}_t, \alpha) \leq \gamma\theta_t^*) \geq 1 - \eta_t^*,
	\end{equation}
	where $\eta_t^* = |A|\delta + |A|N_b \eta_{t+1}^*$, $\eta_T^* = 0$, and
	\begin{equation}
		\begin{aligned}
			\theta_t^* &= (T-t+1)(R_{\max}-R_{\min}) \sqrt{\frac{5\ln(3/\delta)}{\alpha N_b}} \\
			&\quad + \frac{\theta_{t+1}^*}{\alpha}, \quad \theta_T^* = 0.
		\end{aligned}
	\end{equation}
	
	\textbf{Base case ($t = T$):} At the terminal time, the value function equals the immediate cost with no future terms:
	\begin{equation}
		\begin{aligned}
			&V_{M_P,T}^*(\bar{b}_T, \alpha) - \hat{V}_{M_P,T}^*(\bar{b}_T, \alpha) \\
			&= \min_{a \in A} c(\bar{b}_T, a) - \min_{a \in A} c(\bar{b}_T, a) = 0.
		\end{aligned}
	\end{equation}
	The immediate cost $c(\bar{b}_T, a)$ is computed identically in both the true and estimated value functions, hence the base case holds.
	
	\textbf{Inductive step:} Assume the bound holds for $t+1$. Following Theorem~\ref{prf:guarantees_for_PB_MDP_vs_algo_v2}, for any action $a$:
	\begin{equation}
		\begin{aligned}
			&Q_{M_P}^*(\bar{b}_t, a, \alpha) - \hat{Q}_{M_P}^*(\bar{b}_t, a, \alpha) \\
			&= \gamma \big( CVaR_\alpha(V_{M_P}^*(\bar{b}_{t+1}, \alpha)|\bar{b}_t, a) \\
			&\quad - \widehat{CVaR}_\alpha(\{\hat{V}_{M_P}^*(\bar{b}_{t+1}^j, \alpha)\}_{j=1}^{N_b}) \big).
		\end{aligned}
	\end{equation}
	This decomposes into CVaR estimation error and value approximation error as in~\eqref{eq:B_inequality}.
	
	\textbf{(1) CVaR estimation error:} By Theorem~\ref{thm:brown_bounds}, for each $a$:
	\begin{equation}
		\begin{aligned}
			&P \Big( CVaR_\alpha(V_{M_P}^*(\bar{b}_{t+1}, \alpha)|\bar{b}_t, a) \\
			&\quad - \widehat{CVaR}_\alpha(\{V_{M_P}^*(\bar{b}_{t+1}^j, \alpha)\}_{j=1}^{N_b}) \\
			&\leq (T-t+1)(R_{\max}-R_{\min}) \sqrt{\frac{5\ln(3/\delta)}{\alpha N_b}} \Big) \geq 1 - \delta.
		\end{aligned}
	\end{equation}
	By union bound over $|A|$ actions, this holds for all $a$ with probability $\geq 1 - |A|\delta$.
	
	\textbf{(2) Value approximation error:} Using Lemma~\ref{thm:cvar_subtraction_bounds}:
	\begin{equation}
		\begin{aligned}
			&\widehat{CVaR}_\alpha(\{V_{M_P}^*(\bar{b}_{t+1}^j, \alpha)\}_{j=1}^{N_b}) \\
			&- \widehat{CVaR}_\alpha(\{\hat{V}_{M_P}^*(\bar{b}_{t+1}^j, \alpha)\}_{j=1}^{N_b}) \\
			&\leq \frac{1}{N_b \alpha} \sum_{j=1}^{N_b} (V_{M_P}^*(\bar{b}_{t+1}^j, \alpha) - \hat{V}_{M_P}^*(\bar{b}_{t+1}^j, \alpha))^+.
		\end{aligned}
	\end{equation}
	Define the event $E^*$ that the induction hypothesis holds for all $|A| \times N_b$ successor beliefs:
	\begin{equation}
		E^* = \bigcap_{a \in A} \bigcap_{j=1}^{N_b} \{V_{M_P}^*(\bar{b}_{t+1}^{a,j}, \alpha) - \hat{V}_{M_P}^*(\bar{b}_{t+1}^{a,j}, \alpha) \leq \gamma\theta_{t+1}^*\},
	\end{equation}
	where $\bar{b}_{t+1}^{a,j}$ denotes the $j$-th sampled successor belief for action $a$. By union bound, $P(E^*) \geq 1 - |A| N_b \eta_{t+1}^*$. Conditioned on $E^*$, the value approximation error is bounded by $\gamma\theta_{t+1}^*/\alpha$ for all actions.
	
	\textbf{Combining the bounds:} Using the union bound to combine the CVaR estimation error (holding for all actions) and the value approximation error:
	\begin{equation}
		\begin{aligned}
			&P(V_{M_P, t}^*(\bar{b}_t, \alpha) - \hat{V}^*_{M_P, t}(\bar{b}_t, \alpha) \leq \gamma\theta_t^*) \\
			&\geq P\Big( \forall a: Q_{M_P}^*(\bar{b}_t, a, \alpha) - \hat{Q}_{M_P}^*(\bar{b}_t, a, \alpha) \leq \gamma\theta_t^* \Big) \\
			&\geq (1 - |A|\delta) + (1 - |A| N_b \eta_{t+1}^*) - 1 \\
			&= 1 - |A|\delta - |A| N_b \eta_{t+1}^* = 1 - \eta_t^*.
		\end{aligned}
	\end{equation}
	This completes the induction.
	
	\textbf{Closed-form for $\eta_t^*$:}
	Unrolling $\eta_t^* = |A|\delta + |A| N_b \eta_{t+1}^*$ yields $\eta_t^* = |A|\delta \sum_{k=0}^{T-t-1} (|A| N_b)^k$. By geometric series:
	\begin{equation}
		\eta_t^* = |A|\delta \cdot \frac{(|A| N_b)^{T-t} - 1}{|A| N_b - 1}.
	\end{equation}
	
	\textbf{Closed-form for $\theta_t^*$:} Same as Theorem~\ref{prf:guarantees_for_PB_MDP_vs_algo_v2}:
	\begin{equation}
		\theta_t^* = (R_{\max}-R_{\min})\sqrt{\tfrac{5\ln(3/\delta)}{\alpha N_b}} \sum_{k=0}^{T-t-1} \tfrac{T-t-k}{\alpha^k}.
	\end{equation}
	
	\textbf{Final upper bound:} Substituting $\delta \to \delta(|A|N_b - 1)/[|A|((|A|N_b)^{T-t} - 1)]$ to achieve confidence $1-\delta$ yields the theorem statement.
	
	\textbf{Lower bound proof:} The proof follows the same structure. We prove by induction:
	\begin{equation}
		P(V_{M_P,t}^*(\bar{b}_t, \alpha) - \hat{V}_{M_P,t}^*(\bar{b}_t, \alpha) \geq \lambda_t^*) \geq 1 - \psi_t^*,
	\end{equation}
	where $\psi_t^* = |A|\delta + |A| N_b \psi_{t+1}^*$, $\psi_T^* = 0$, and
	\begin{equation}
		\begin{aligned}
			\lambda_t^* &= -\frac{\gamma(T-t+1)(R_{\max}-R_{\min})}{\alpha} \sqrt{\frac{\ln(1/\delta)}{2 N_b}} \\
			&\quad + \frac{\lambda_{t+1}^*}{\alpha}, \quad \lambda_T^* = 0.
		\end{aligned}
	\end{equation}
	
	\textbf{(1) CVaR error (lower):} By Theorem~\ref{thm:brown_bounds} with union bound over $|A|$ actions:
	\begin{equation}
		\begin{aligned}
			&P \Big( \forall a: CVaR_\alpha(V_{M_P}^*(\bar{b}_{t+1}, \alpha)|\bar{b}_t, a) \\
			&\quad - \widehat{CVaR}_\alpha(\{V_{M_P}^*(\bar{b}_{t+1}^j, \alpha)\}_{j=1}^{N_b}) \\
			&\geq -\frac{(T-t+1)(R_{\max}-R_{\min})}{\alpha} \sqrt{\frac{\ln(1/\delta)}{2 N_b}} \Big) \geq 1 - |A|\delta.
		\end{aligned}
	\end{equation}
	
	\textbf{(2) Value error (lower):} By Lemma~\ref{thm:cvar_subtraction_bounds}:
	\begin{equation}
		\begin{aligned}
			&\widehat{CVaR}_\alpha(\{V_{M_P}^*(\bar{b}_{t+1}^j, \alpha)\}_{j=1}^{N_b}) \\
			&- \widehat{CVaR}_\alpha(\{\hat{V}_{M_P}^*(\bar{b}_{t+1}^j, \alpha)\}_{j=1}^{N_b}) \\
			&\geq -\frac{1}{N_b \alpha} \sum_{j=1}^{N_b} (\hat{V}_{M_P}^*(\bar{b}_{t+1}^j, \alpha) - V_{M_P}^*(\bar{b}_{t+1}^j, \alpha))^+.
		\end{aligned}
	\end{equation}
	Define the event $F^*$ that the induction hypothesis holds for all $|A| \times N_b$ successor beliefs:
	\begin{equation}
		F^* = \bigcap_{a \in A} \bigcap_{j=1}^{N_b} \{V_{M_P}^*(\bar{b}_{t+1}^{a,j}, \alpha) - \hat{V}_{M_P}^*(\bar{b}_{t+1}^{a,j}, \alpha) \geq \lambda_{t+1}^*\}.
	\end{equation}
	By union bound, $P(F^*) \geq 1 - |A| N_b \psi_{t+1}^*$. Conditioned on $F^*$, the value error is bounded by $\lambda_{t+1}^*/\alpha$.
	
	\textbf{Combining:} $P(V_{M_P,t}^*(\bar{b}_t, \alpha) - \hat{V}_{M_P,t}^*(\bar{b}_t, \alpha) \geq \lambda_t^*) \geq 1 - \psi_t^*$.
	
	\textbf{Closed-form:} $\psi_t^* = |A|\delta \cdot ((|A| N_b)^{T-t} - 1)/(|A| N_b - 1)$, and
	\begin{equation}
		\lambda_t^* = -\frac{\gamma(R_{\max}-R_{\min})}{\alpha}\sqrt{\frac{\ln(1/\delta)}{2 N_b}} \sum_{j=0}^{T-t-1} \frac{T-t+1-j}{\alpha^j}.
	\end{equation}
	
	\textbf{Final lower bound:} Substituting $\delta$ to achieve confidence $1-\delta$ yields the theorem statement. For $N_b = 1$, the geometric series simplifies accordingly.
\end{proof}

\begin{theorem}\label{thm:cvar_subtraction_bounds}(Bound between two empirical CVaR estimators)
	Let $\alpha\in(0,1)$ and let
	$\hat f_x$ and $\hat f_y$ be empirical CVaR estimators defined from
	samples $x=(x_1,\dots,x_n)\in\mathbb{R}^n$ and
	$y=(y_1,\dots,y_n)\in\mathbb{R}^n$ by
	\begin{equation}
		\hat f_x
		\triangleq
		\inf_{r\in\mathbb{R}}
		\left\{
		r+\frac{1}{n\alpha}\sum_{i=1}^n (x_i-r)^+
		\right\},
	\end{equation}
	\begin{equation}
		\hat f_y
		\triangleq
		\inf_{r\in\mathbb{R}}
		\left\{
		r+\frac{1}{n\alpha}\sum_{i=1}^n (y_i-r)^+
		\right\}.
	\end{equation}
	Then the following bounds hold:
	\begin{equation}
		-\frac{1}{n\alpha}\sum_{i=1}^n (y_i-x_i)^+
		\;\le\;
		\hat f_x-\hat f_y
		\;\le\;
		\frac{1}{n\alpha}\sum_{i=1}^n (x_i-y_i)^+.
	\end{equation}
\end{theorem}

\begin{proof}
	Applying Lemma~\ref{thm:basic_cvar_subtraction_bounds} with arguments
	$(x,y)$ yields
	\begin{equation}
		\hat f_x-\hat f_y
		\le
		\frac{1}{n\alpha}\sum_{i=1}^n (x_i-y_i)^+.
	\end{equation}
	Exchanging the roles of $x$ and $y$ in the same lemma gives
	\begin{equation}
		\hat f_y-\hat f_x
		\le
		\frac{1}{n\alpha}\sum_{i=1}^n (y_i-x_i)^+,
	\end{equation}
	which is equivalent to
	\begin{equation}
		\hat f_x-\hat f_y
		\ge
		-\frac{1}{n\alpha}\sum_{i=1}^n (y_i-x_i)^+.
	\end{equation}
	Combining the two inequalities completes the proof.
\end{proof}

\begin{lemma}\label{thm:basic_cvar_subtraction_bounds}[Asymmetric bound for empirical CVaR]
	Let $\alpha\in(0,1)$ and define $f:\mathbb{R}^n\to\mathbb{R}$ by
	\begin{equation}
		f(x_1,\dots,x_n)
		\triangleq
		\inf_{r\in\mathbb{R}}
		\left\{
		r+\frac{1}{n\alpha}\sum_{i=1}^n (x_i-r)^+
		\right\}.
	\end{equation}
	Then, for all $x,y\in\mathbb{R}^n$,
	\begin{equation}
		f(x)-f(y)
		\le
		\frac{1}{n\alpha}\sum_{i=1}^n (x_i-y_i)^+.
	\end{equation}
\end{lemma}

\begin{proof}
	By definition of $f$, for any $r\in\mathbb{R}$,
	\begin{equation}
		f(x)
		\le
		r+\frac{1}{n\alpha}\sum_{i=1}^n (x_i-r)^+,
		\quad
		f(y)
		\le
		r+\frac{1}{n\alpha}\sum_{i=1}^n (y_i-r)^+.
	\end{equation}
	Subtracting the two inequalities yields
	\begin{equation}
		f(x)-f(y)
		\le
		\frac{1}{n\alpha}
		\sum_{i=1}^n
		\bigl[(x_i-r)^+-(y_i-r)^+\bigr].
	\end{equation}
	
	For any $a,b,r\in\mathbb{R}$, the inequality
	\begin{equation}
		(a-r)^+-(b-r)^+\le (a-b)^+
	\end{equation}
	holds. Applying this inequality componentwise gives
	\begin{equation}
		(x_i-r)^+-(y_i-r)^+
		\le
		(x_i-y_i)^+,
		\quad i=1,\dots,n.
	\end{equation}
	
	Substituting into the previous expression, we obtain
	\begin{equation}
		f(x)-f(y)
		\le
		\frac{1}{n\alpha}\sum_{i=1}^n (x_i-y_i)^+.
	\end{equation}
	Since the right-hand side does not depend on $r$, the bound holds uniformly over all $r$, completing the proof.
\end{proof}

\begin{theorem}\label{thm:brown_bounds}
	(Results of \cite{brown2007large}): If $\text{supp}(X) \subseteq [a, b]$ and $X$ has a continuous distribution function, then for any $\delta \in (0, 1]$, 
	\begin{equation}
		P\Bigl{(}CVaR(X)-\hat{C}(X)>(b-a)\sqrt{\frac{5ln(3/\delta)}{\alpha n}}\Bigr{)}\leq \delta,
	\end{equation}
	\begin{equation}
		P\Bigl{(}CVaR(X)-\hat{C}(X)<-\frac{(b-a)}{\alpha}\sqrt{\frac{ln(1/\delta)}{2n}}\Bigr{)}\leq \delta.
	\end{equation}
\end{theorem}

\begin{algorithm}[htbp]
	\caption{ICVaR-POMCPOW}
	\label{alg:pomcpow_elaboareted_diff}
	\textbf{Input}: state $s$, history $h$, depth $d$ \\
	\textbf{Global Variables}: $N, Q, B, W, \mathcal{C}, M, k_o, \alpha_o, \gamma$ \\
	\textbf{Output}: total return
	\begin{algorithmic}[1]
		\STATE \textbf{Simulate($s, h, d$)}
		\IF{$d = 0$}
		\STATE \textbf{return} $0$
		\ENDIF
		\STATE $a \gets \text{ActionProgWiden}(h)$
		\STATE $(s', o, c) \sim G(s, a)$
		\STATE \alghl{$\text{old\_weights\_sum} \gets \sum_{o\in \mathcal{C}(ha)} W(hao)$ $\triangleright$ \text{New code}}
		\IF{$|\mathcal{C}(ha)| \leq k_o N(ha)^{\alpha_o}$}
		\STATE $M(hao) \gets M(hao) + 1$
		\ELSE
		\STATE $o \gets \text{select } o \in \mathcal{C}(ha)$ \text{w.p. $M(hao)/\sum M(hao)$}
		\ENDIF
		\STATE Append $s'$ to $B(hao)$
		\STATE Append $P(o \mid s, a, s')$ to $W(hao)$
		\IF{$o \notin \mathcal{C}(ha)$}
		\STATE $\mathcal{C}(ha) \gets \mathcal{C}(ha) \cup \{o\}$
		\STATE \KP{$total \gets c+\gamma Rollout(s',hao,d-1)$ $\triangleright$ \text{Old code}}
		\STATE \alghl{\text{No rollout is performed } $\triangleright$ \text{New code}}
		\ELSE
		\STATE $s' \gets \text{select } B(hao)[i]$ \text{ w.p. $W(hao)[i]/\sum_j W(hao)[j]$}
		\STATE $c \gets -R(s, a, s')$
		\STATE \KP{$total \gets c + \gamma\text{Simulate}(s', hao, d-1)$ $\triangleright$ \text{Old code}}
		\STATE \alghl{No recursive call for return estimation $\triangleright$ \text{New code}}
		\ENDIF
		\STATE \alghl{SIMULATE(s', hao, d - 1) $\triangleright$ \text{New code}}
		\STATE $N(ha) \gets N(ha) + 1$
		\STATE $N(h) \gets N(h) + 1$
		\STATE \alghl{$\text{new\_weights\_sum} \gets \text{old\_weights\_sum}+P(o \mid s, a, s')$ $\triangleright$ \text{New code}}
		\STATE \alghl{$Imm(ha) \gets \frac{Imm(ha)\times \text{old\_weights\_sum} + c * P(o \mid s, a, s')}{\text{new\_weights\_sum}}$ $\triangleright$ \text{New code}}
		\STATE \KP{$Q(ha) \gets Q(ha) + \frac{total - Q(ha)}{N(ha)}$ $\triangleright$ \text{Old code}}
		\STATE \alghl{$Q(ha) \gets Imm(ha) + \widehat{CVaR}_\alpha(\{hao\}_{o\in \mathcal{C}(ha)})$ $\triangleright$ \text{New code}}
		\STATE \KP{\textbf{return} $total$ $\triangleright$ \text{Old code}}
		\STATE \alghl{\textbf{return} \text{None} $\triangleright$ \text{New code}}
		\STATE \textbf{End Simulate} 
	\end{algorithmic}
\end{algorithm}

\subsection{ICVaR-PFT-DPW}

We present ICVaR-PFT-DPW in Algorithm \ref{alg:pft-dpw_elaboareted_diff}.

\begin{algorithm}[tb]
	\caption{PFT-DPW}
	\label{alg:pft-dpw_elaboareted_diff}
	\begin{algorithmic}[1]
		\STATE \textbf{Function} \textsc{Plan}($b$)
		\FOR{$i = 1$ to $n_0$}
		\STATE \textsc{Simulate}($b, d_{\max}$)
		\ENDFOR
		\STATE \textbf{return} $\arg\min_a Q(ba)$
		\STATE \textbf{End Function}
		\STATE
		\STATE \textbf{Function} \textsc{Simulate}($b, d$)
		\IF{$d = 0$}
		\STATE \textbf{return} $0$
		\ENDIF
		\STATE $a \gets \textsc{ActionProgWiden}(b)$
		\IF{$|C(ba)| \le k_o \, N(ba)^{\alpha_o}$}
		\STATE $(b', c) \gets G_{\text{PF}}(m, ba)$
		\STATE $C(ba) \gets C(ba) \cup \{(b', c)\}$
		\STATE \KP{$total \gets c + \gamma \cdot \textsc{Rollout}(b', d-1)$ $\triangleright$ \text{Old code}} 
		\ELSE
		\STATE $(b', c) \gets$ sample uniformly from $C(ba)$
		\STATE \KP{$total \gets c + \gamma \cdot \textsc{Simulate}(b', d-1)$ $\triangleright$ \text{Old code}}
		\ENDIF
		\STATE \alghl{SIMULATE(hao, d - 1) $\triangleright$ \text{New code}}
		\STATE $N(b) \gets N(b) + 1$
		\STATE $N(ba) \gets N(ba) + 1$
		\STATE \KP{$Q(ba) \gets Q(ba) + \dfrac{total - Q(ba)}{N(ba)}$ $\triangleright$ \text{Old code}}
		\STATE \alghl{$Q(ha) \gets c + \widehat{CVaR}_\alpha(\{hao\}_{o\in \mathcal{C}(ha)})$ $\triangleright$ \text{New code}}
		\STATE \KP{\textbf{return} $total$ $\triangleright$ \text{Old code}}
		\STATE \alghl{\textbf{return} None $\triangleright$ \text{New code}}
		\STATE \textbf{End Function}
	\end{algorithmic}
\end{algorithm}

\section{Simulations}\subsection{Laser Tag POMDP}
The LaserTag POMDP is a pursuit-evasion problem where a robot agent must navigate a discrete grid environment to tag an adversarial opponent while receiving noisy observations about the opponent's location. The state space consists of the robot's position, opponent's position, and a terminal flag, with the robot having five discrete actions: four movement directions (North, South, East, West) and a tag action. The robot's movement is deterministic but constrained by walls, while the opponent follows a stochastic policy that moves toward the robot with 0.4 probability in the x-direction, 0.4 probability in the y-direction, and 0.2 probability of staying in place. The robot receives 8-dimensional continuous observations representing laser range measurements in cardinal and diagonal directions, corrupted by Gaussian noise with standard deviation 1.0. The reward structure provides positive reward for successful tagging, negative reward for failed tag attempts, and step costs for movement actions. Episodes terminate when the robot successfully tags the opponent, creating a challenging partially observable planning problem where the agent must maintain beliefs about the opponent's location while navigating strategically to achieve the tagging objective.

\subsection{2D Light-Dark POMDP}
The Continuous Light-Dark POMDP is a navigation problem where an agent must navigate through a continuous 2D space to reach a goal location while dealing with position-dependent observation noise and avoiding obstacles. The state space consists of continuous 2D position vectors, with the agent having either continuous movement actions or discrete directional actions (up, down, left, right). The agent's movement is stochastic with Gaussian noise added to the intended movement vector, while observations of the agent's position are corrupted by distance-dependent noise that decreases when near light beacons scattered throughout the environment. The reward structure provides positive rewards for reaching the goal region, negative penalties for obstacle collisions and movement costs, with multiple reward model variants available including standard, decaying hit probability, and dangerous states models. Episodes terminate when the agent reaches the goal, hits an obstacle, or moves outside the grid boundaries, creating a challenging partially observable navigation problem where the agent must balance exploration and exploitation while maintaining beliefs about its true position under noisy observations and leveraging light beacons to improve localization accuracy.

\subsection{Experiment Configuration}

This section details the hyperparameters used in the experiments.

\paragraph{Common Planner Parameters.}
All planners use a planning time budget of 4 seconds per step, planning horizon $T = 10$, risk parameter $\alpha = 0.1$, and confidence parameter $\delta = 0.05$. The exploration constant is set to $c = T(R_{\max} - R_{\min})$. ICVaR is estimated using the policy evaluation algorithm with branching factor $N_b = 5$ and evaluation horizon $T = 3$. Results are averaged over 200 episodes.

\paragraph{LaserTag Environment.}
The LaserTag environment is a pursuit-evasion grid world with discrete state and action spaces and continuous observations. Key parameters: discount factor $\gamma = 0.99$, tag reward $200$, transition error probability $0.2$, dangerous area penalty $5$, dangerous area radius $1.0$, and 7 dangerous areas. The belief is represented with 20 particles. Episodes run for 25 steps.

\paragraph{LaserTag Planner Parameters.}
Progressive widening parameters: $k_a = 5$, $k_o = 5$, $\alpha_a = 0.0$, $\alpha_o = 0.5$. Planning depth: $T = 10$.

\paragraph{LightDark Environment.}
The LightDark environment is a continuous navigation domain with continuous state, action, and observation spaces. Key parameters: discount factor $\gamma = 0.5$, goal reward $100$, fuel cost $2$ per action, obstacle penalty $-30$, goal radius $1.5$, obstacle radius $1.4$, grid size $10 \times 10$. Observations are obtained from 9 beacons with radius $1.0$. State transition covariance: $0.1 I_2$, observation covariance: $0.15 I_2$. The belief is represented with 20 particles. Episodes run for 20 steps.

\paragraph{LightDark Planner Parameters.}
Progressive widening parameters: $k_a = 8$, $k_o = 10$, $\alpha_a = 0.0$, $\alpha_o = 0.01$. Planning depth: $T = 10$. Actions are sampled uniformly from the unit circle with 50\% probability of selecting cardinal directions.

\end{document}